%% file: asyn_sa_arxiv_v2.tex
\documentclass[10pt, letterpaper, twoside]{article}

\input{macros-asyn-sa-arxiv.tex}

\begin{document} 
\markboth{Stability in Asynchronous Stochastic Approximation}{Stability in Asynchronous Stochastic Approximation}

\title{A Note on Stability in Asynchronous Stochastic Approximation without Communication Delays\thanks{This research was supported in part by DeepMind and Amii.}}

\author[1]{\normalsize Huizhen Yu}
\author[1,2]{\normalsize Yi Wan}
\author[1,3]{\normalsize Richard S. Sutton}

\affil[1]{\normalsize Department of Computing Science, University of Alberta, Canada}
\affil[2]{\normalsize Meta AI, USA}
\affil[3]{\normalsize Alberta Machine Intelligence Institute (Amii), Canada}
\date{} 

\maketitle

\blfootnote{{\it Email addresses:} \texttt{janey.hzyu@gmail.com} (Huizhen Yu; corresponding author), \texttt{yiwan@meta.com} (Yi Wan), \texttt{rsutton@ualberta.ca} (Richard S. Sutton)}

\vspace*{-1.2cm}
%\bigskip

\noindent{\bf Abstract:} In this paper, we study asynchronous stochastic approximation algorithms without communication delays. Our main contribution is a stability proof for these algorithms that extends a method of Borkar and Meyn by accommodating more general noise conditions. We also derive convergence results from this stability result and discuss their application in important average-reward reinforcement learning problems.


\smallskip
\section{Introduction}

In this paper, we study the stability and convergence of a family of asynchronous stochastic approximation (SA) algorithms that hold significant importance in reinforcement learning (RL) applications. These algorithms operate in a finite-dimensional space $\R^d$ and, given an initial vector $x_0 \in \R^d$, iteratively compute $x_n \in \R^d$ for $n \geq 1$ using an asynchronous scheme. This asynchrony involves selective updates to individual components at each iteration. To be specific, at the start of iteration $n \geq 0$, a nonempty subset $Y_n \subset \I : = \{ 1, 2, \ldots, d\}$ is randomly selected according to some mechanism. The $i$th component $x_n(i)$ of $x_n$ is then updated according to the following rule: $x_{n+1}(i) = x_n(i)$ if $i \not\in Y_n$; and  
\begin{equation} \label{eq-gen-form}
x_{n+1}(i) = x_n(i) + \beta_{n,i} \left(h_i(x_n) + \omega_{n+1} (i) \right), \quad \text{if} \ i \in Y_n.
\end{equation}
This process involves a diminishing random stepsize $\beta_{n,i}$, a Lipschitz continuous function $h : \R^d \to \R^d$ expressed as $h = (h_1, \ldots, h_d)$, and a random noise term $\omega_{n+1} = (\omega_{n+1}(1), \ldots, \omega_{n+1}(d)) \in \R^d$.

The key components, including the sets $Y_n$, satisfy a set of conditions, which we detail in Section~\ref{sec-prel}. These conditions are similar to those introduced by Borkar \cite{Bor98,Bor00} for asynchronous SA in terms of the stepsizes and $Y_n$. The function $h$ satisfies the stability criterion established in the seminal work of Borkar and Meyn \cite{BoM00}. However, our requirements on noise terms $\{\omega_n\}$ are more general than considered in these prior works (cf.\ Remark~\ref{rmk-cond}b). Specifically, we assume that $\omega_{n+1} = M_{n+1} + \epsilon_{n+1}$, where $\{M_{n+1}\}$ forms a martingale difference sequence subject to specific conditional variance conditions, while $\epsilon_{n+1}$ is such that $\| \epsilon_{n+1} \| \leq \delta_{n+1} ( 1 + \| x_n \|)$ with $\delta_{n+1} \to 0$ almost surely (a.s.), as $n \to \infty$.

Importantly, although this type of noise is standard for convergence analysis (as seen in Borkar \cite[Chap.\ 2.2]{Bor09}) when the algorithm is deemed stable (i.e., $\{x_n\}$ is bounded a.s.), to the best of our knowledge, it has not been considered in the stability analysis of asynchronous SA within the Borkar--Meyn framework \cite{BoM00}. This omission presents a notable limitation, given the fundamental role of stability in SA analysis \cite{Bor09,KuY03}.

Our main contribution in this paper is a stability proof that extends Borkar and Meyn's method \cite{BoM00} by accommodating the more general noise conditions. This stability result (Theorem~\ref{thm-1}), combined with arguments from Borkar \cite{Bor98,Bor09}, then leads to convergence results (Theorems \ref{thm-2} and \ref{thm-3}) for the SA algorithms under consideration. 

As mentioned, these results have important applications in RL. In particular, a class of SA algorithms fitting within the considered algorithmic framework is known as average-reward Q-learning in RL. These algorithms use a stochastic relative value iteration approach to solve finite-state-and-action Markov or semi-Markov decision processes (MDPs or SMDPs) under the average-reward optimality criterion (Abounadi, Bertsekas, and Borkar \cite{ABB01}; Wan, Naik, and Sutton \cite{WNS21a,WNS21b}). We have applied the results of this paper to establish the convergence of several such Q-learning algorithms in a separate work (Wan, Yu, and Sutton \cite{WYS24}). In that work, we consider weakly communicating MDPs/SMDPs, a broader problem class than previously addressed in \cite{ABB01,WNS21a,WNS21b}, and we have shown that the conditions for applying the results of this paper are met. 
Crucially, the more general noise conditions are essential for addressing the SMDP case \cite{WNS21b,WYS24}, where the function $h$ is determined by expected holding times, whose estimates from data approach true values asymptotically.
In this paper, we will discuss the specializations of our convergence results to the average-reward Q-learning context (see Cor.~\ref{cor-ql} and the discussion around it).

The paper is organized as follows. Section~\ref{sec-prel} presents the algorithmic framework, the main stability and convergence theorems, and a preliminary analysis. Section~\ref{sec-stab} presents the stability proof, followed by the proofs of convergence results and their specialization to the RL context in Section~\ref{sec-cvg}. Section~\ref{sec-conc-rmks} concludes with several remarks. An alternative stability proof under a stronger noise condition from the prior works \cite{Bor98,BoM00} is provided in the \hyperref[app-alt-stab]{Appendix}. 

\section{Algorithmic Framework, Main Results, and Preliminary Analysis} \label{sec-prel}

We start by providing a precise description of the algorithmic framework and the required conditions, as outlined in the introduction.
Let $\{\alpha_n\}_{n \geq 0}$ be a given positive sequence of diminishing stepsizes.  Let $\ind \{ E\}$ denote the indicator for an event $E$. Consider an asynchronous SA algorithm of the following form: At iteration $n \geq 0$, $x_{n+1}(i) = x_n(i)$ for $i \not\in Y_n$; and for $i \in Y_n$,
\begin{equation} \label{eq-alg0}
    x_{n+1}(i)  = x_n(i)  + \alpha_{\nu(n,i)} \big( h_i (x_n) + M_{n+1}(i) + \epsilon_{n+1}(i) \big),
\end{equation}
where $\nu(n,i) \= \sum_{k=0}^n \ind \{ i \in Y_k\}$, the cumulative number of updates to the $i$th component prior to iteration $n$.  
The algorithm is associated with an increasing family of $\sigma$-fields, denoted by $\{\F_n\}_{n \geq 0}$, where each $\F_n \supset \sigma (x_m, Y_m, M_m, \epsilon_m; m \leq n)$. The following conditions will apply consistently throughout the paper, and will not be explicitly stated in intermediate results.

\begin{assumption}[Conditions on the function $h$] \label{cond-h} \hfill 
\begin{enumerate} 
\item[{\rm (i)}] $h$ is Lipschitz continuous; i.e., for some $L \geq 0$, $\| h(x) - h(y) \| \leq L \| x - y\|$ for all $x, y \in \R^{d}$.
\item[{\rm (ii)}] For $c \geq 1$, define $h_c(x) \= h(cx)/c$. As $c \to \infty$, the function $h_c(x)$ converges uniformly on compact subsets of $\R^{d}$ to a continuous function $h_\infty$. Furthermore, the ODE
$ \dot{x}(t) = h_\infty (x(t)) $
has the origin as its unique globally asymptotically stable equilibrium.
\end{enumerate}
\end{assumption}

\begin{assumption}[Conditions on noise terms $M_n, \epsilon_n$] \label{cond-ns} \hfill  
\begin{enumerate} 
\item[{\rm (i)}] For all $n \geq 0$, $\E [ \| M_{n+1} \| ] < \infty$, $\E [ M_{n+1} \mid \F_n ] = 0$ a.s., and moreover, for some deterministic constant $K \geq 0$, 
$ \E[ \| M_{n+1} \|^2 \mid \F_n ] \leq K (1 +\| x_n \|^2 )$ a.s.
\item[{\rm (ii)}] For all $n \geq 0$, $\| \epsilon_{n+1} \| \leq \delta_{n+1} ( 1 + \| x_n \|)$, where $\delta_{n+1}$ is $\F_{n+1}$-measurable and as $n \to \infty$, $\delta_n \to 0$ a.s. 
\end{enumerate}
\end{assumption} 

\begin{assumption}[Stepsize conditions]  \label{cond-ss} \hfill
\begin{enumerate} 
\item[{\rm (i)}] $\sum_n \alpha_n = \infty$, $\sum_n \alpha_n^2 < \infty$, and 
$\alpha_{n+1} \leq \alpha_n$ for all $n$ sufficiently large.
\item[{\rm (ii)}] For $x \in (0,1)$, 
$\sup_n \frac{\alpha_{[ x n]}}{ \alpha_n} < \infty$,
where $[x n]$ denotes the integral part of $xn$.
\item[{\rm (iii)}] For $x \in (0,1)$, as $n \to \infty$, $\frac{ \sum_{k=0}^{[ y n ]} \alpha_k }{ \sum_{k=0}^{n} \alpha_k} \to 1$ uniformly in $y \in [x, 1]$.
\end{enumerate}
\end{assumption} 

For $x > 0$, let $N(n,x) \= \min  \left\{ m > n : \sum_{k = n}^m \alpha_k \geq x \right\}$.
\begin{assumption}[Asynchronous update conditions] 
\label{cond-us} \hfill 
\begin{enumerate}
\item[{\rm (i)}] There exists a deterministic constant $\Delta > 0$ such that 
$\liminf_{n \to \infty} \nu(n,i)/n  \geq \Delta$ a.s., for all $i \in \I$.
\item[\rm (ii)] For each $x > 0$, the limit $\lim_{n \to \infty} \frac{ \sum_{k = \nu(n,i)}^{\nu(N(n,x), i)} \alpha_k}{ \sum_{k = \nu(n,j)}^{\nu(N(n,x), j)} \alpha_k}$ exists a.s., for all $i, j \in \I$.
\end{enumerate}
\end{assumption} 

\begin{remark}[About the algorithmic conditions] \label{rmk-cond} \rm \hfill \\*[1pt]
(a) Assumption~\ref{cond-h} on $h$ is the stability criterion introduced by Borkar and Meyn \cite{BoM00}. Note that the functions $h_c$ and $h_\infty$, like $h$, are also Lipschitz continuous with modulus $L$, and $h_\infty(0) = 0$, under this assumption.\\*[2pt]
(b) For asynchronous SA, the noise terms considered in the prior works \cite{Bor98,BoM00,ABB01} satisfy the general conditions given in Assumption~\ref{cond-ns}, but are assumed to be more specific: $\epsilon_n = 0$ and $M_n = F(x_{n-1}, \zeta_n)$, where $\{\zeta_n\}_{n \geq 1}$ are exogenous, independent, and identically distributed (i.i.d.) random variables, and $F$ is a function uniformly Lipschitz in its first argument. We will discuss these noise conditions further in Section~\ref{sec-conc-rmks} and the \hyperref[app-alt-stab]{Appendix}.\\*[2pt]
(c) Assumptions~\ref{cond-ss} and~\ref{cond-us} regarding stepsizes and asynchrony are largely the same as those used in average-reward Q-learning \cite{ABB01}. These conditions, with some minor variations in Assumption~\ref{cond-us}(ii), were originally introduced in the broader context of asynchronous SA in Borkar \cite{Bor98,Bor00} for this particular stepsize structure $\alpha_{\nu(n,i)}$. The purpose of these conditions, reflected in Lemmas~\ref{lem4} and~\ref{lem-cvg-2}, is to create partial asynchrony, aligning the asynchronous algorithm's asymptotic behavior, on average, with that of a synchronous one.\\*[2pt]
(d) This partial asynchrony is important for average-reward Q-learning applications. While Q-learning can achieve stability and convergence in fully asynchronous schemes (either of the form (\ref{eq-gen-form}) or more general) for discounted-reward MDPs and total-reward MDPs of the stochastic shortest path type \cite{Tsi94,YuB13}, these analyses do not extend to average-reward Q-learning.
\qed
\end{remark}

\subsection{Stability and Convergence Theorems} \label{sec-2.1}

We now state our main results: 
\begin{theorem}[Stability] \label{thm-1}
Under Assumptions~\ref{cond-h}--\ref{cond-us}, the sequence $\{x_n\}$ generated by algorithm (\ref{eq-alg0}) is bounded a.s.
\end{theorem}

\begin{theorem}[Convergence] \label{thm-2}
Under Assumptions~\ref{cond-h}--\ref{cond-us}, the sequence $\{x_n\}$ generated by algorithm (\ref{eq-alg0}) converges a.s.\ to a (possibly sample path-dependent) compact, connected, internally chain transitive, invariant set of the ODE $\dot{x}(t) = h(x(t))$.
\end{theorem}

For the definitions of internally chain transitive sets and invariant sets for ODEs, see Borkar \cite[Chap.\ 2.1]{Bor09}.

These two theorems parallel the results \cite[Chap.\ 3, Theorem 7 and Chap.\ 2, Theorem 2]{Bor09} for synchronous SA algorithms. 
We will prove them in subsequent sections. Additionally, our convergence analysis will yield a refined version of Theorem~\ref{thm-2}, characterizing the algorithm's asymptotic behavior in terms of segments of consecutive iterates rather than single iterates (see Theorem~\ref{thm-3} and Cor.~\ref{cor-ql}).

\subsection{Preliminary Analysis} \label{sec-prel-ana}

In our stability and convergence analyses, we adopt an ODE-based approach. We work with continuous trajectories formed by linear interpolations of the iterates $\{x_n\}$ and establish connections between these trajectories and solutions of non-autonomous ODEs of the form $\dot{x}(t) = \lambda(t) g(x)$, where the function $g$ can take various forms, such as $h$ and $h_c$ ($c \geq 1$), 
depending on the context. The characteristics of these ODE solutions will be examined in our main analysis. This subsection's primary focus is on the time-dependent components $\lambda$. We provide definitions and derive their asymptotic properties, which will be needed in our subsequent analysis.

The continuous trajectory associated with $\{x_n\}$ will be constructed differently for stability and convergence analyses. In our stability proof, we use the deterministic stepsize $\alpha_n$ as the elapsed time between the $n$th and $(n+1)$th iterates, while in our convergence proof, we opt for a random stepsize for technical convenience. (Using random stepsizes in stability analysis seems non-viable under our noise conditions.) The corresponding functions $\lambda$ thus also differ in the two cases. Let us start by discussing the approach used in stability analysis.

Define a continuous trajectory $\bar x(t)$ by linearly interpolating $\{x_n\}$ as follows: 
Let $t(0) \= 0$ and $t(n) \= \sum_{k=0}^{n-1} \alpha_k$, $n \geq 1$.
Define, for $n \geq 0$,
$\bar x(t(n)) \= x_n$ and 
$$  \bar x(t) \=  x_n +  \tfrac{t - t(n)}{t(n+1) - t(n)} \, ( x_{n+1} - x_n), \ \  t \in (t(n), t(n+1)).$$
To define $\lambda(\cdot)$, we will first rewrite algorithm (\ref{eq-alg0}) explicitly in terms of the stepseizes $\{\alpha_n\}$ used to define the temporal coordinates of the trajectory $\bar x(t)$: for $i \in \I$, 
\begin{equation} \label{eq-alg}
    x_{n+1}(i)  = x_n(i) + \alpha_n \, \q(n, i) \left( h_i (x_n) + M_{n+1}(i) + \epsilon_{n+1}(i) \right),
\end{equation}
where $\q(n,i) \= \frac{\alpha_{\nu(n,i)}}{\alpha_n} \ind \{i \in Y_n\}$.
 
\begin{lemma}  \label{lem1}
For some deterministic constant $C > 0$, 
it holds almost surely that
$\max_{i \in \I} \q(n,i) \leq C$ and $\sum_{ i \in \I} \q(n,i) \geq 1$ for all $n \geq \bar n$,
where $\bar n \geq 0$ is a sufficiently large integer depending on each sample path.
\end{lemma}

\begin{proof}
As discussed in \cite[p.\ 842]{Bor98}, Assumption~\ref{cond-ss}(ii), together with Assumption~\ref{cond-ss}(i) on $\{\alpha_n\}$ being eventually nonincreasing, implies that for $x \in (0,1)$,
$\sup_n \sup_{y \in [x, 1]} \frac{\alpha_{[ y n]}}{ \alpha_n} < \infty.$
By Assumption~\ref{cond-us}(i), almost surely, $\min_{i \in \I} \nu(n,i)/n \geq \Delta'$ for all $n$ sufficiently large, where $\Delta' \=\Delta/2 \in (0,1)$ is a deterministic constant. 
Thus, for the finite deterministic constant $C \=  \sup_n \sup_{y \in [\Delta', 1]} \frac{\alpha_{[ y n]}}{ \alpha_n}$,
it holds almost surely that $\max_{i \in \I} \q(n,i) \leq \max_{i \in \I} \frac{\alpha_{\nu(n,i)}}{\alpha_n} \leq C$ for all $n$ sufficiently large. 
Since $\{\alpha_n\}$ is eventually nonincreasing by Assumption~\ref{cond-ss}(i) and the sets $Y_n$ are nonempty, Assumption~\ref{cond-us}(i) also implies that almost surely,
$\sum_{ i \in I} \q(n,i) =  \sum_{ i \in Y_n} \frac{\alpha_{\nu(n,i)}}{\alpha_n}  \geq 1$ for all $n$ sufficiently large.
\end{proof}

Define a diagonal matrix-valued, piecewise constant function $\lambda(t)$ as follows: For $t \in [t(n), t(n+1))$, $n \geq 0$, 
\begin{equation} \label{eq-def-lambda}
    \lambda(t) \=\text{diag} \big( \, \q(n, 1) \!\wedge\! C, \, \q(n, 2) \!\wedge\! C, \, \ldots, \, \q(n, d) \!\wedge\! C \, \big),
\end{equation} 
where $a \wedge b : = \min \{a, b\}$.
For any fixed $\bar t \geq 0$, we view $\lambda(\bar t + \cdot)$ as an element in the space $\Upsilon$ which comprises all Borel-measurable functions that map $t \geq 0$ to a $d \times d$ diagonal matrix with nonnegative diagonal entries bounded by $C$. More precisely, two such functions $\lambda', \lambda''$ are regarded as the same element in $\Upsilon$, if with respect to (w.r.t.) the Lebesgue measure, $\lambda'(t) = \lambda''(t)$ almost everywhere (a.e.).
As in Borkar \cite{Bor98} and \cite[Chap.\ 7.2]{Bor09}, we equip the space $\Upsilon$ with the coarsest topology that makes the mappings $\psi_{t,f}: \lambda' \mapsto \int_0^t  \lambda'(s) f(s) \, ds$ continuous for all $t > 0$ and $f \in L_2([0, t]; \R^{d})$ (the space of all $\R^{d}$-valued functions on $[0,t]$ that are square-integrable w.r.t.\ to the Lebesgue measure). 
This means that the family of open sets in $\Upsilon$ consists of arbitrary unions of finite intersections of sets of the form
$\left\{   \lambda' \in \Upsilon :   \left\| \psi_{t, f}(\lambda') - \psi_{t, f}(\bar \lambda) \right\| < \epsilon \right\},$
where $\epsilon, t  > 0$, $\bar \lambda \in \Upsilon$, $f \in L_2([0, t]; \R^{d})$. (These sets will be useful in our stability analysis, particularly, in the proof of Lemma~\ref{lem6}.)
With this topology, $\Upsilon$ is a compact metrizable space (by the Banach-Alaoglu theorem and the separability of the Hilbert spaces $L_2([0, t]; \R^{d})$, $t > 0$;  cf.\ \cite[Chaps.\ 7.2 and 11.1.2]{Bor09}).
Thus any sequence in $\Upsilon$ contains a convergent subsequence.

\begin{lemma} \label{lem4}
Almost surely, for any sequence $t_n \geq 0$ with $t_n \uparrow \infty$, all limit points of the sequence $\{\lambda(t_n + \cdot)\}_{n \geq 0}$ in $\Upsilon$ have the form
$\lambda^*(t) = \rho(t) I$,
where $I$ is the identity matrix and $\rho(\cdot)$ is a real-valued Borel-measurable function satisfying $\tfrac{1}{d} \leq \rho(t) \leq C$ for all $t \geq 0$.
\end{lemma}

\begin{proof}
Let $\{t^1, t^2, \ldots\}$ be a dense set in $\R_+$. 
Consider a sample path for which Assumption~\ref{cond-us}(i) holds and Assumption~\ref{cond-us}(ii) holds for all $x \in \{t^1, t^2, \ldots\}$. (Note that such sample paths form a set of probability $1$.) By its proof, Lemma~\ref{lem1} holds for such a sample path. 

Given $\{t_n\}$ with $t_n \uparrow \infty$, consider any subsequence $\{\lambda(t_{n_k} + \cdot)\}_{k \geq 0}$ converging to some 
$\lambda^* \in \Upsilon$. 
Let $i, j \in \I$.
With Assumptions~\ref{cond-ss} and~\ref{cond-us} holding, it follows from Lemma~\ref{lem1} and the reasoning given in the proofs of \cite[Thm.\ 3.2]{Bor98} and \cite{Bor00} (this is where Assumption~\ref{cond-ss}(iii) is used) that
\begin{equation} \label{eq-prf-lambda}
  \int_0^{t^\ell} \lambda^*_{ii}(s) ds =   \int_0^{t^\ell} \lambda^*_{jj}(s) ds,  \quad \forall \, \ell \geq 1, \ \forall \, i, j \in \I.
\end{equation}   
Since $\{t^\ell\}$ is dense in $\R_+$ and $\lambda^*_{ii}(s) \in [0, C]$, it follows that $f(t) \= \int_{0}^{t} \lambda^*_{ii}(s) ds$ defines the same function $f$ for any $i \in \I$ and hence $\lambda^*_{ii}(s) = \lambda^*_{jj}(s)$ a.e.\ by the Lebesgue differentiation theorem \cite[Thm.\ 7.2.1]{Dud02}.
Since functions in $\Upsilon$ that are identical a.e.\ are treated as the same function, we have $\lambda^*(t) = \rho(t) I$ for some Borel-measurable function $\rho$ with $\rho(t) \in [0,C]$. It remains to show $\rho(t) \geq 1/d$ a.e. By the convergence $\lambda(t_{n_k} + \cdot) \to \lambda^*$ in $\Upsilon$, for all $t, s > 0$,
$$  \int_t^{t+s} \!\rho(y) \, \text{trace}( I) dy = \lim_{k \to \infty} \int_t^{t+s} \!\text{trace}(\lambda(t_{n_k} + y)) dy \geq s,$$
where the inequality follows from Lemma~\ref{lem1} and the definition of $\lambda(\cdot)$. Thus $\int_t^{t+s} \!\rho(y) dy \geq  \tfrac{s}{d}$ for all $t, s > 0$, implying $\rho(t) \geq \tfrac{1}{d}$ a.e.\ by the Lebesgue differentiation theorem \cite[Thm.\ 7.2.1]{Dud02}. 
\end{proof}

\begin{remark} \rm \label{rmk-2}
We make two comments on the preceding proof:\\*[1pt]
(a) The proofs of Borkar \cite[Thm.\ 3.2]{Bor98} and \cite{Bor00} ingeniously employ L'H\^{o}pital's rule. While these proofs deal with a function $\lambda(\cdot)$ different from ours, the same reasoning is applicable in our case. It shows that under Assumptions~\ref{cond-ss} and~\ref{cond-us}, for each $x > 0$, all these limits in Assumption~\ref{cond-us}(ii), $\lim_{n \to \infty} \frac{ \sum_{k = \nu(n,i)}^{\nu(N(n,x), i)} \alpha_k}{ \sum_{k = \nu(n,j)}^{\nu(N(n,x), j)} \alpha_k}$, $i, j \in \I$, must equal to $1$ a.s. This leads to (\ref{eq-prf-lambda}).\\*[2pt]
(b) In the application of the Lebesgue differentiation theorem \cite[Thm.\ 7.2.1]{Dud02}, alternative measure-theoretical arguments can be employed. Given that $\int_{t}^{t'} \lambda^*_{ii}(s) ds =   \int_{t}^{t'} \lambda^*_{jj}(s) ds$ for all $0 \leq t < t'$, both $\lambda^*_{ii}(s) ds$ and $\lambda^*_{jj}(s) ds$ define the same $\sigma$-finite measure on $\R_+$ according to \cite[Thm.\ 3.2.6]{Dud02}. Consequently, $\lambda^*_{ii}(s) = \lambda^*_{jj}(s)$ a.e.\ by the Radon-Nikodym theorem \cite[Thm.\ 5.5.4]{Dud02}. Given that $\int_t^{t+s} \!\rho(y) dy \geq  \tfrac{s}{d}$ for all $t, s > 0$, by a differentiation theorem for measures \cite[Chap.\ VII, \secmark8]{Doo53}, $\rho(t) \geq \tfrac{1}{d}$ a.e.
\qed
\end{remark}

We now describe the setup that will be used for convergence analysis. In this case, we write algorithm (\ref{eq-alg0}) equivalently as: for all $i \in \I$,
\begin{equation} \label{eq-alg1}
    x_{n+1}(i)  = x_n(i) + \tl \alpha_n \,  \tl \q(n, i) \left( h_i (x_n) + M_{n+1}(i) + \epsilon_{n+1}(i) \right),
\end{equation}
where $\tl \alpha_n \= \sum_{i \in Y_n}  \alpha_{\nu(n, i)}$, $\tl \q(n, i) \= \frac{\alpha_{\nu(n, i)}}{\tl \alpha_n} \ind\{i \in Y_n\}$, and thus $\sum_{i \in \I} \tl \q(n,i) = 1$. Correspondingly, let 
\begin{equation} \label{def-conv-ana-ode-time}
 \tl t(0) \= 0, \quad \textstyle{\tl t(n) \= \sum_{k=0}^{n -1} \tl \alpha_k,}  \ \ \ n \geq 1.
\end{equation}  
We define $\bar x(t)$ to be linear interpolations of $\{x_n\}$ as before, but with $\tl t(n)$ replacing $t(n)$; in other words, we use the random stespsizes $\{\tl \alpha_n\}$ as elapsed time between consecutive iterates $\{x_n\}$. This choice is motivated by the simpler limiting behavior of the resulting function $\lambda(t)$, which we denote by $\tl \lambda(t)$ in this case. 
Specifically, $\tilde \lambda(\cdot)$, a diagonal matrix-valued, piecewise constant trajectory, is given by: for  $t \in [\tl t(n), \tl t(n+1))$, $n \geq 0$,
\begin{equation} \label{eq-tlambda}
 \tl{\lambda}(t) \=\text{diag} \big( \, \tl \q(n, 1), \, \tl \q(n, 2), \, \ldots, \, \tl \q(n, d)  \, \big).
\end{equation}   

We view $\tilde \lambda(\cdot)$ as an element in the space $\tilde \Upsilon$ which comprises all Borel-measurable functions that map $t \geq 0$ to a $d \times d$ diagonal matrix with nonnegative diagonal entries summing to $1$. Regarded as a subset of $\Upsilon$ with the relative topology, $\tilde \Upsilon$ is a compact metrizable space. Thus any sequence $\{\tl \lambda(t_n + \cdot)\}_{n \geq 0}$ contains a convergent subsequence in $\tilde \Upsilon$. Furthermore, when $t_n \to \infty$, the limit point is unique and corresponds to a constant function. To establish this fact, we prove a technical result first.

Define $\tilde N(n, x) \= \min \left\{ m > n : \sum_{k=n}^{m} \sum_{i \in Y_k} \alpha_{\nu(k, i)} \geq x \right\}$ for $x > 0$.
Given Assumptions~\ref{cond-ss} and~\ref{cond-us}(i), the following lemma is equivalent to an assumption introduced and employed in Borkar \cite{Bor00} in lieu of Assumption~\ref{cond-us}(ii) [cf.\ Remark~\ref{rmk-3}]. Here we obtain this lemma as a consequence of Assumption~\ref{cond-us} and Lemma~\ref{lem4}.

\begin{lemma} \label{lem-cvg-1}
For each $x > 0$, 
$\lim_{n \to \infty} \frac{ \sum_{k = \nu(n,i)}^{\nu(\tilde N(n,x), i)} \alpha_k}{ \sum_{k = \nu(n,j)}^{\nu(\tilde N(n,x), j)} \alpha_k} = 1$ a.s., for all $i, j \in \I$.
\end{lemma}

\begin{proof}
Consider a sample path for which Assumption~\ref{cond-us} and Lemma~\ref{lem4} hold. Fix $x > 0$. 
By the definition of $\tilde N(n,x)$ and the fact $\alpha_n \to 0$ as $n \to 0$, we have
$$\sum_{i \in \I} \sum_{k = \nu(n,i)}^{\nu(\tilde N(n,x), i)} \alpha_k =  \sum_{k=n}^{\tilde N(n,x)} \sum_{i \in Y_k} \alpha_{\nu(k, i)} \to x \  \ \text{as $n \to \infty$}, $$ 
so the assertion of this lemma is equivalent to that 
\begin{equation} \label{eq-lc1-prf1}
   \lim_{n \to \infty}  \sum_{k = \nu(n,i)}^{\nu(\tilde N(n,x), i)} \alpha_k  = x/d,  \quad \forall \, i \in \I.
\end{equation}   
To prove (\ref{eq-lc1-prf1}), it suffices to show that any increasing sequence $\{n_\ell \}_{\ell \geq 1}$ of natural numbers has a subsequence $\{n'_\ell \}_{\ell \geq 1}$ along which $\sum_{k = \nu(n'_\ell,i)}^{\nu(\tilde N(n'_\ell,x), i)} \alpha_k  \to x/d$ as $\ell \to \infty$, for all $i \in \I$.
To this end, let $t_{n_\ell} \= t(n_\ell)$ and let $\lambda^*(\cdot) = \rho(\cdot) I$ be a limit point of the sequence $\{\lambda(t_{n_\ell} + \cdot)\}_{\ell \geq 1}$ in $\Upsilon$ (cf.\ Lemma~\ref{lem4}). We denote an associated convergent subsequence again by $\{\lambda(t_{n_\ell} + \cdot)\}_{\ell \geq 1}$, to simplify notation. Thus $\lambda(t_{n_\ell} + \cdot) \to \lambda^*$ and we need to prove $\sum_{k = \nu(n_\ell,i)}^{\nu(\tilde N(n_\ell,x), i)} \alpha_k  \to x/d$ as $\ell \to \infty$, for all $i \in \I$.

Choose $\epsilon \in (0, x)$. Since $\rho(s) \in [1/d, C]$ for all $s \geq 0$ by Lemma~\ref{lem4}, the two equations below define uniquely two constants $\underline{\tau} > 0$ and $\bar \tau > 0$, respectively:
$$  \int_0^{\underline{\tau}} \rho(s) ds =  \frac{x - \epsilon}{d}, \quad \int_0^{\bar{\tau}} \rho(s) ds = \frac{x + \epsilon}{d}.$$
Then, by the convergence of $\lambda(t_{n_\ell} + \cdot) \to \lambda^*$, for all $i \in \I$, as $\ell \to \infty$,
$$ \int_0^{\underline{\tau}} \lambda_{ii}(t_{n_\ell} + s) ds  \to  \frac{x - \epsilon}{d}, \ \ \ \int_0^{\bar{\tau}} \lambda_{ii}(t_{n_\ell} + s) ds  \to  \frac{x + \epsilon}{d}.$$
In view of Lemma~\ref{lem1} and the definition of $\lambda$ [cf.\ \eqref{eq-def-lambda}], this implies that 
\begin{equation} 
 \underline{c}_\ell(i)   \=  \, \sum_{k = \nu(n_\ell, i)}^{\nu(N(n_\ell,  \underline{\tau}), i)} \alpha_k \, \to  \, \frac{x - \epsilon}{d}, \qquad
  \bar c_\ell(i)    \=  \, \sum_{k = \nu(n_\ell, i)}^{\nu(N(n_\ell,  \bar{\tau}),  i)} \alpha_k \, \to \, \frac{x + \epsilon}{d}, \label{eq-lc1-prf2}
\end{equation} 
and hence 
\begin{align} 
  \sum_{k=n_\ell}^{N(n_\ell,\underline{\tau})} \sum_{i \in Y_k} \alpha_{\nu(k, i)} & =  \sum_{i \in \I} \underline{c}_\ell(i) \, \to \, x - \epsilon, \label{eq-lc1-prf3a} \\    \sum_{k=n_\ell}^{N(n_\ell,\bar \tau)} \sum_{i \in Y_k} \alpha_{\nu(k, i)} & = \sum_{i \in \I} \bar{c}_\ell(i) \, \to \, x + \epsilon. \label{eq-lc1-prf3b}
\end{align}
From (\ref{eq-lc1-prf3a})-(\ref{eq-lc1-prf3b}) and the definition of $\tilde N(n, x)$, it follows that for all $\ell$ sufficiently large,
$N(n_\ell,  \underline{\tau}) <  \tilde N(n_\ell, x) <   N(n_\ell,  \bar{\tau}).$
Consequently, for all $\ell$ sufficiently large,
\begin{equation} \label{eq-lc1-prf4}
 \underline{c}_\ell(i)  \leq   \sum_{k = \nu(n_\ell,  i)}^{\nu(\tilde N(n_\ell,x), i)} \alpha_k \leq \bar c_\ell(i), \quad \forall \, i \in \I.
 \end{equation}
 Using (\ref{eq-lc1-prf4}), (\ref{eq-lc1-prf2}), and the arbitrariness of $\epsilon$, we obtain that
 $ \sum_{k = \nu(n_\ell, i)}^{\nu(\tilde N(n_\ell,x), i)} \alpha_k \to x/d$ for all $i \in \I$, as $\ell \to \infty$. This proves (\ref{eq-lc1-prf1}), establishing the lemma.
 \end{proof}
 
By combining Lemma~\ref{lem-cvg-1} with the reasoning presented in the proofs of \cite[Thm.\ 3.2]{Bor98} and \cite{Bor00} (see also Remark~\ref{rmk-2}a), we obtain the following lemma concerning the limit of $\tl \lambda(t + \cdot)$ as $t \to \infty$. The proof details are similar to those for Lemma~\ref{lem4} and are therefore omitted.
 
\begin{lemma}  \label{lem-cvg-2}
Almost surely,  as $t \to \infty$, $\tl \lambda(t + \cdot)$ converges in $\tl \Upsilon$ to the constant function $\bar \lambda(\cdot) \equiv \tfrac{1}{d} I$.
\end{lemma}

\begin{remark} \rm \label{rmk-3}
(a) The condition employed in Borkar \cite{Bor00} instead of Assumption~\ref{cond-us}(ii) is that for each $x > 0$, $\lim_{n \to \infty} \frac{ \sum_{k = \nu(n,i)}^{\nu(\tilde N(n,x), i)} \alpha_k}{ \sum_{k = \nu(n,j)}^{\nu(\tilde N(n,x), j)} \alpha_k}$ exists a.s., for all $i, j \in \I$. Under Assumptions~\ref{cond-ss} and~\ref{cond-us}(i), these limits, if they exist, have to equal $1$ (cf.\ the proof of \cite[Thm.\ 3.2]{Bor98}), so this condition is equivalent to Lemma~\ref{lem-cvg-1}, as mentioned earlier.\\*[2pt]
(b) Although we will not need this fact in our subsequent analysis, we can also show that under Assumptions~\ref{cond-ss} and~\ref{cond-us}(i), the preceding condition from \cite{Bor00} implies Assumption~\ref{cond-us}(ii) (hence, in view of Lemma~\ref{lem-cvg-1}, the two are equivalent). The proof is similar to that of Lemma~\ref{lem-cvg-1}, but with the roles of $\tilde N(n,\cdot)$ and $N(n,\cdot)$ reversed, using Lemma~\ref{lem-cvg-2} instead of Lemma~\ref{lem4}, and using also the compactness of the space $\Upsilon$.
\qed
\end{remark}

\section{Stability Analysis} \label{sec-stab}
In this section, we prove Theorem~\ref{thm-1} on the boundedness of the iterates $\{x_n\}$. Employing the method introduced by Borkar and Meyn \cite{BoM00} and recounted in the book by Borkar \cite[Chap.\ 3.2]{Bor09}, we study scaled iterates and relate their asymptotic behavior to solutions of specific limiting ODEs involving the function $h_\infty$. Our proof follows a structure similar to the stability analysis in \cite[Chap.\ 3.2]{Bor09} for synchronous algorithms and is divided into two sets of intermediate results. The first group, presented in Section~\ref{sec-stab-prf1}, shows how scaled iterates progressively `track' solutions of ODEs with corresponding scaled functions $h_c$. The second group, in Section~\ref{sec-stab-prf2}, establishes a stability-related solution property for these ODEs as the scale factor $c$ tends to infinity. With these results in place, our proof then concludes similarly to the approach in \cite[Chap.\ 3.2]{Bor09}.

Throughout our proof, we will heavily rely on results from \cite[Chap.\ 3.2]{Bor09} to keep our explanation succinct and avoid repetition. Our main emphasis will be on elements that differ from the derivations in \cite[Chap.\ 3.2]{Bor09}, yet are essential to our treatment of the asynchronous algorithm.

\subsection{Relating Scaled Iterates to ODE Solutions Involving\\ Scaled Functions $h_c$} \label{sec-stab-prf1}

Consider algorithm \eqref{eq-alg0} in its equivalent form \eqref{eq-alg} and the corresponding continuous trajectory $\bar x(t)$ defined before \eqref{eq-alg}. Following the approach of \cite{BoM00} and \cite[Chap.\ 3.2]{Bor09} for stability analysis, we work with a scaled version of $\bar x(\cdot)$. To define it, let $T > 0$. With $m(0) \=0$ and $T_0 \= 0$, define recursively, for $n \geq 0$,
\begin{equation} \label{eq-def-tm}
   m(n+1) \= \min \{ m : t(m) \geq T_n + T \}, \quad T_{n+1} \= t\big(m(n+1)\big).
\end{equation}
Here $t(m) = \sum_{k=0}^{m-1} \alpha_k$ with $t(0) = 0$ as defined previously.
This divides the time axis $[0, \infty)$ into intervals $[T_n, T_{n+1})$, $n \geq 0$, whose lengths are about $T$ and approach $T$ as $n \to \infty$. For simplicity in the expressions of certain bounds that involve $\sup_n \alpha_n$, we assume $\sup_n \alpha_n \leq 1$ in what follows, so that each interval is at most $T+1$ in length. The choice of the arbitrary value of $T$ will be determined later based on the properties of the ODEs involved in the proof. 

We proceed by defining a piecewise linear function $\hat x(\cdot)$ by scaling $\bar x(t)$ as follows. Let $a \vee b : = \max \{ a, b\}$.
For each $n \geq 0$, with $r(n) \= \| x_{m(n)} \| \vee 1$, define
\begin{equation} \label{eq-hx0}
\hat x(t) \= \bar x(t) / r(n)   \ \ \ \text{for } t \in [T_n, T_{n+1}).
\end{equation}
Note that $\hat x(\cdot)$ can have `jumps' at times $T_1, T_2, \ldots$. For notational convenience, when analyzing the behavior of $\hat x(t)$ on the semi-closed interval $[T_n, T_{n+1})$, we introduce a `copy' denoted by $\hat x^n(t)$ defined on the closed interval $[T_n, T_{n+1}]$ as follows:
$\hat x^n(t) \= \hat x(t)$ for $t \in [T_n, T_{n+1})$, and $\hat x^n(T_{n+1}) \= \hat x(T_{n+1}^-) \= \lim_{t \uparrow T_{n+1}} \hat x(t)$.

A key intermediate result necessary to relate the scaled trajectory $\hat x(t)$ to specific ODE solutions is to establish $\sup_t \| \hat x(t) \| < \infty$. For synchronous algorithms, this is done in \cite[Chap.\ 3.2]{Bor09} in several steps, starting with the proof of 
$ \sup_{t} \E [ \| \hat x(t) \|^2] < \infty$. It is achieved by deriving the bound \cite[Chap.\ 3, Lem.\ 4]{Bor09}
\begin{equation} \label{eq-bd-hatx}
 \E \left[ \| \hat x^n(t(k+1)) \|^2 \right]^{\frac{1}{2}} \leq e^{K_1(T+1)} ( 1 + K_2 (T+1) ), 
\end{equation} 
for all $m(n) \leq k < m(n+1)$, where $K_1, K_2$ are constants independent of $n$.
In the asynchronous case here, we will take a similar approach. However, by \eqref{eq-alg} and the definition of $\hat x(\cdot)$, we have that for $k$ with $m(n) \leq k < m(n+1)$,
\begin{align}  \label{eq-hx}
  \hat x^n(t(k+1)) & = \hat x^n(t(k)) + \alpha_k \Lambda_k h_{r(n)}(\hat x^n(t(k)))  
    + \alpha_k \Lambda_k \hat M_{k+1} + \alpha_k \Lambda_k \hat \epsilon_{k+1},
\end{align}  
where 
$\hat M_{k+1} \= M_{k+1} / r(n)$, $\hat \epsilon_{k+1} \= \epsilon_{k+1}/r(n)$, and 
$$  \Lambda_k \= \text{diag} \big( \q(k, 1), \q(k, 2), \ldots, \q(k, d) \big).$$
While $\{\q(k,i)\}_{k \geq 0,i \in \I}$ are bounded a.s.\ (Lemma~\ref{lem1}), they need not be bounded by a deterministic constant; while by Assumption~\ref{cond-ns}(ii) $\|\hat \epsilon_{k+1} \| \leq \delta_{k+1} ( 1 + \| \hat x^n(t(k)) \|)$ with $\delta_{k+1} \to 0$ a.s., there is no requirement on the conditional variance of $\delta_{k+1}$. These prevent us from applying directly the proof arguments of \cite[Chap.\ 3, Lem.\ 4]{Bor09}. 

To work around this issue, we now introduce `better-behaved' auxiliary processes $\tilde x^n(t)$ on $[T_n, T_{n+1}]$ for $n \geq 0,$ which have more favorable analytical properties and will, almost surely, eventually coincide with $\hat x^n(\cdot)$ as $n \to \infty$. Let $C$ be the constant given by Lemma~\ref{lem1}, and fix $\bar a > 0$. 
For $n \geq 0$, let
\begin{align} 
k_{n,1}  &  \= \min \left\{  k : \max_{i \in \I} \q(k, i) > C, \, m(n) \leq k < m(n+1) \right\}, \notag \\ 
 k_{n,2}  & \= \min \left\{  k :  \delta_k  > \bar a, \  m(n)+1 \leq k \leq m(n+1) \right\},  \notag 
\end{align}
with $k_{n,1} \= \infty$ and $k_{n,2} \= \infty$ if the sets in their respective defining equations are empty. Let $k_n \= k_{n,1} \wedge k_{n,2}$.
For $n \geq 0$, define $\tilde x^n(\cdot)$ on $[T_n, T_{n+1}]$ as follows. 
Let $\tilde x^n(T_n) \= \hat x(T_n)$. For $m(n) \leq k < m(n+1)$, let
\begin{align}  \label{eq-tx}
 \tilde x^n(t(k+1)) & = \tilde x^n(t(k)) + \alpha_k \tilde \Lambda_k h_{r(n)}(\tilde x^n(t(k)))  + \alpha_k \tilde M_{k+1} + \alpha_k \tilde \epsilon_{k+1},
\end{align}   
where $\tilde \Lambda_k \= \ind\{k < k_n\} \Lambda_k$, $\tilde M_{k+1}  \=  \tl \Lambda_k \hat M_{k+1}$, and 
\begin{equation} \label{eq-txa}
  \tilde \epsilon_{k+1} \= \tl \Lambda_k \cdot \ind\{k+1 < k_{n,2}\} \hat \epsilon_{k+1}.
\end{equation}
Finally, on the interval $(t(k), t(k+1))$, let $\tl x^n(t)$ be the linear interpolation between $\tl x^n(t(k))$ and $\tl x^n(t(k+1))$.
As can be seen, in each time interval $[T_n, T_{n+1}]$, 
\begin{align} 
\tilde x^n(t) = \tl x^n (t(k_n)),  \quad  \text{if} \  t \geq t(k_n); \qquad \ \ \, \notag \\
 \tilde x^n(t)   = \hat x^n (t) \ \  \text{for} \ t \leq t(k_n),  \ \ \ \text{if} \ k_n = k_{n,1} < k_{n,2};  \label{eq-tx-hx1} \\
   \tilde x^n(t)  = \hat x^n (t) \ \  \text{for} \ t \leq t(k_n-1),  \    \text{if} \ k_n = k_{n,2} \leq k_{n,1}.  \label{eq-tx-hx2}
\end{align}
Moreover, since $\max_{i \in \I} \q(n, i) \leq C$ and $\delta_n \leq \bar a$ for sufficiently large $n$ (a.s.) by Lemma~\ref{lem1} and Assumption~\ref{cond-ns}(ii), we have $\tl x^n(\cdot) = \hat x^n(\cdot)$ for sufficiently large $n$ (a.s.).

By definition $k_{n,1}, k_{n,2},$ and $k_n$ are stopping times w.r.t.\ $\{\F_k\}$, so $\ind\{k < k_n\}$ is $\F_k$-measurable and $\ind\{k+1 < k_{n,2}\}$ is $\F_{k+1}$-measurable. Hence $\tl \Lambda_k$ is $\F_k$-measurable, whereas $\tl M_{k+1}$ and $\tl \epsilon_{k+1}$ are $\F_{k+1}$-measurable. Also by the definition of $k_n$, for $m(n) \leq k < k_n$, the diagonal entries of the matrix $\Lambda_k$ are all bounded by $C$. Let $\bar C$ be an upper bound on $\| \Lambda \|$ for all such diagonal matrices $\Lambda$, where $\| \Lambda\|$ is the matrix norm w.r.t.\ the norm on $\R^{d}$. 
Then
\begin{equation} \label{eq-Lambda}
   \| \tilde \Lambda_k \| \leq \bar C,  \qquad \forall \,  m(n) \leq k < m(n+1).
\end{equation}   
Moreover, by the construction of $\tilde x^n$ in \eqref{eq-tx}-\eqref{eq-txa}, we have the following.

\begin{lemma} \label{lem2a}
For $n \geq 0$ and all $k$ with $m(n) \leq k < m(n+1)$, we have $\E [ \| \tilde M_{k+1} \| ] < \infty$, $\E [ \tilde M_{k+1} \nmid \F_k ] = 0$ a.s., and
\begin{align} 
   \E [ \| \tilde M_{k+1} \|^2 \mid \F_k ]  & \leq \bar C^2 K ( 1 + \| \tilde x^n(t(k)) \|^2) \ \ a.s., \label{eq-prf-noise1} \\
     \| \tilde \epsilon_{k+1} \|  & \leq \bar C (\bar a \wedge \delta_{k+1}) ( 1 + \| \tilde x^n(t(k)) \|) \ \ a.s., \label{eq-prf-noise2}
\end{align}
where $K$ is the constant in Assumption~\ref{cond-ns}(i).
\end{lemma}

\begin{proof}
For $m(n) \leq k < m(n+1)$, since $\hat M_{k+1} = M_{k+1}/r(n)$ and $\hat x^n(t(k)) = x_k/r(n)$, where $r(n) \geq 1$, we have, by Assumption~\ref{cond-ns}(i) on $M_{k+1}$, that 
$$\E [ \| \hat M_{k+1} \| ] < \infty, \ \ \  \E [ \hat M_{k+1} \nmid \F_k ] = 0 \ \ \text{and} \ \ \E [ \| \hat M_{k+1} \|^2 \nmid \F_k ] \leq K ( 1 + \| \hat x^n(t(k)) \|^2) \ \  a.s.$$
Since $\tilde M_{k+1} = \tl \Lambda_k \hat M_{k+1}$, by \eqref{eq-Lambda}, we have
$\E [ \| \tilde M_{k+1} \| ] < \infty$, $\E [ \tilde M_{k+1} \mid \F_k ] = 0$ a.s., and moreover, by the definition of $\tl \Lambda_k$,
\begin{equation} \label{eq-l2a-prf1}
\E [ \| \tilde M_{k+1} \|^2 \mid \F_k ] \leq \bar C^2 \ind\{k < k_n\} \, \E [ \| \hat M_{k+1} \|^2 \mid \F_k ].
\end{equation}
If $k < k_n$, $\tilde x^n(t(k)) = \hat x^n(t(k))$ by the definition of $\tilde x^n(\cdot)$ [cf.\ (\ref{eq-tx-hx1})-(\ref{eq-tx-hx2})]. Therefore, on $\{k < k_n\}$, a.s.,
\begin{equation} 
 \E [ \| \hat M_{k+1} \|^2 \!\mid \F_k ] \leq K ( 1 + \| \hat x^n(t(k)) \|^2) = K ( 1 + \| \tl x^n(t(k)) \|^2), \notag
 \end{equation} 
which together with (\ref{eq-l2a-prf1}) proves (\ref{eq-prf-noise1}).
Finally, (\ref{eq-prf-noise2}) is a direct consequence of (\ref{eq-Lambda}), Assumption~\ref{cond-ns}(ii), and the definitions of $k_{n,2}$ and $\tl \epsilon_{k+1}$ [cf.\ (\ref{eq-txa})]. 
\end{proof}

 By (\ref{eq-tx}) and (\ref{eq-Lambda}), for all $m(n) \leq k < m(n+1)$, 
 \begin{align} \label{eq-lem1-prf1}
  \| \tilde x^n(t(k+1)) \| & \leq \| \tilde x(t(k))\|  + \alpha_k \bar C \, \| h_{r(n)}(\tilde x^n(t(k))) \|  + \alpha_k \| \tilde M_{k+1}\|  + \alpha_k \| \tilde \epsilon_{k+1} \|.
 \end{align}
 Using (\ref{eq-lem1-prf1}) and Lemma~\ref{lem2a}, we can now follow, step by step, the proof arguments for \cite[Chap.\ 3, Lem.\ 4]{Bor09} to derive the following bound, analogous to the bound \eqref{eq-bd-hatx}: For all $k$ with $m(n) \leq k < m(n+1)$,
$$  \E \left[ \| \tilde x^n(t(k+1)) \|^2 \right]^{\frac{1}{2}} \leq e^{K_1(T+1)} ( 1 + K_2 (T+1) ),$$
for some suitable constants $K_1, K_2$ independent of $n$. 
With this bound, we obtain Lemma~\ref{lem2}(i) below. It then immediately leads to Lemma~\ref{lem2}(ii), by applying Lemma~\ref{lem2a} to $\{\tl M_{k}\}_{k \geq 1}$ and a convergence theorem for square-integrable martingales \cite[Prop.\ VII-2-3(c)]{Nev75} to the martingale $\sum_{k=0}^{m-1} \alpha_k \tilde M_{k+1}, m \geq 1$, similarly to the proof of \cite[Chap.\ 3, Lem.\ 5]{Bor09}. 
 
\begin{lemma} \label{lem2} The following hold:
\begin{enumerate}
\item[\rm (i)] $\sup_{n \geq 0} \sup_{t \in [T_n, T_{n+1}]} \E [ \| \tilde x^n(t) \|^2 ] < \infty$.
\item[\rm (ii)] Almost surely, $\tilde \zeta_m \= \sum_{k=0}^{m-1} \alpha_k \tilde M_{k+1}$ converges in $\R^{d}$ as $m \to \infty$.
\end{enumerate}
\end{lemma}

Next, we use Lemma~\ref{lem2} to relate the trajectories $\tilde x^n(\cdot)$ and $\hat x^n(\cdot)$, $n \geq 0$, to solutions of non-autonomous ODEs that involve the scaled functions $h_{r(n)}$ and the trajectory $\lambda(\cdot)$ defined by \eqref{eq-def-lambda}.
In particular, let $x^n: [T_n, T_{n+1}] \to \R^{d}$ be the unique solution of the ODE
\begin{equation} \label{eq-ode0}
\dot{x}(t) = \lambda(t) h_{r(n)} (x(t)), \ \ \ t \in [T_n, T_{n+1}],
\end{equation}
with $x^n(T_n) = \hat x(T_n) = x_{m(n)}/r(n)$.

\begin{lemma} \label{lem3} The following hold for $\tilde x^n(\cdot), n \geq 0$:
\begin{enumerate}
\item[\rm (i)] $\sup_{n \geq 0} \sup_{t \in [T_n, T_{n+1}]} \| \tilde x^n(t) \|  < \infty$ a.s.; 
\item[\rm (ii)] $\lim_{n \to \infty} \sup_{t \in [T_n, T_{n+1}]} \left\| \tilde x^n(t) - x^n(t) \right\| = 0$ a.s. 
\end{enumerate}
\noindent Hence the same conclusions hold also for $\hat x^n(\cdot), n \geq 0$.
\end{lemma}

For $\tilde x^n(\cdot)$, the proofs of both parts of Lemma~\ref{lem3} use Lemma~\ref{lem2}(ii), (\ref{eq-tx}) and (\ref{eq-Lambda}), as well as Assumption~\ref{cond-h} on $h$ and (\ref{eq-prf-noise2}) given in Lemma~\ref{lem2a}.  
For Lemma~\ref{lem3}(ii), which is a consequence of Lemma~\ref{lem3}(i), we also use the fact that almost surely, $k_n = \infty$ for all $n$ sufficiently large and in (\ref{eq-prf-noise2}) $\delta_{k} \to 0$ as $k \to \infty$ by Lemma~\ref{lem1} and Assumption~\ref{cond-ns}(ii).
The proof details are very similar to those for \cite[Chap.\ 3, Lem.\ 6 and Chap.\ 2, Lem.\ 1]{Bor09} and are therefore omitted.
We then obtain the last assertion of Lemma~\ref{lem3} for $\hat x^n(\cdot)$ because almost surely, $\hat x^n(\cdot)$ coincides with $\tilde x^n(\cdot)$ for all $n$ sufficiently large, as discussed earlier.

Lemma~\ref{lem3} has established that as $n \to \infty$,
the trajectories $\hat x^n(\cdot)$ `track' the solutions of the ODEs \eqref{eq-ode0}.
We will now further investigate the solution properties of these ODEs if $r(n)$ were to become exceedingly large, and use these properties to complete the stability proof.

\subsection{Stability in Scaling Limits of Corresponding ODEs and\\ Proof Completion} \label{sec-stab-prf2}

Let $\unitS$ denote the unit sphere in $\R^{d}$. For each $n$, the solution $x^n(\cdot)$ of the ODE \eqref{eq-ode0} is determined by the functions $h_{r(n)}$ and $\lambda(T_n + \cdot)$, and an initial condition within the unit ball of $\R^{d}$. If $r(n)$ becomes sufficiently large, the initial condition lies on $\unitS$, and $h_{r(n)}$ approaches the function $h_\infty$ due to Assumption~\ref{cond-h}(ii). According to Lemma~\ref{lem4}, almost surely, any limit point of the sequence $\{ \lambda(T_n + \cdot)\}_{n \geq 0}$ in $\Upsilon$ has the form $\lambda^*(t) = \rho(t) I$ with $\tfrac{1}{d} \leq \rho(t) \leq C$ for all $t \geq 0$.
This leads us to consider \emph{an arbitrary function $\lambda^* \in \Upsilon$ of the above form} and the associated limiting ODE
\begin{equation} \label{eq-lim-ode}
    \dot{x}(t) = \lambda^*(t) \, h_\infty (x(t)) \quad \text{with} \ \ x(0) \in \unitS.
\end{equation}    
We now proceed to derive stability-related solution properties for such ODEs and their `nearby' ODEs.

For $x \in \unitS$, denote by $\phi_\infty(t ; x)$ the unique solution of (\ref{eq-lim-ode}) with initial condition $\phi_\infty(0; x) = x$.
This solution is related by time scaling to $\phi^o_\infty(t; x)$, the unique solution of the ODE $\dot{x}(t) = h_\infty (x(t))$ with initial condition $\phi^o_\infty(0; x) = x$; in particular,
\begin{equation} \label{eq-rel-phi}
   \phi_\infty(t; x) = \phi^o_\infty(\tau(t); x), \ \ \text{where} \ \ \tau(t) = \int_0^t \rho(s) ds.
\end{equation}
Under Assumption~\ref{cond-h} on $h_\infty$, there exists $T^o > 0$ such that for all initial conditions $x \in \unitS$, $\| \phi^o_\infty(t; x) \| < 1/8$ for all $t \geq T^o$ \cite[Chap.\ 3, Lem.\ 1]{Bor09}.
Since $\rho(t) \geq \tfrac{1}{d}$, this means that for all $t \geq T^o d$, $\|\phi_\infty(t; x) \| = \| \phi^o_\infty(\tau(t); x) \| < 1/8$. So we have the lemma below, establishing the stability of the limiting ODEs of the form (\ref{eq-lim-ode}).

\begin{lemma} \label{lem5}
There exists $T > 0$, independent of the choice of $\lambda^*$, such that for all $x \in \unitS$, $\| \phi_\infty(t; x) \| < 1/8$ for all $t \geq T$.
\end{lemma}

The next three results (Lemma~\ref{lem6}, Cor.~\ref{cor1}, and Lemma~\ref{lem7}) aim to extend this stability property to `nearby' ODEs within a certain time horizon. 
These results generalize \cite[Chap.\ 3, Lem.\ 2 and Cor.\ 3]{Bor09} from the synchronous context, where a single limiting ODE is involved, to the asynchronous context, which involves potentially multiple limiting ODEs. 

For $c \geq 1$ and $\lambda' \in \Upsilon$, let $\phi_{c, \lambda'}(t ; x)$ denote the unique solution of the ODE 
\begin{equation}
 \dot{x}(t) = \lambda'(t) \, h_c (x(t))  \quad \text{with} \ \ \phi_{c, \lambda'}(0; x) = x.
\end{equation} 
Recall the constant $\bar C$ from \eqref{eq-Lambda} and the constant $L$, which bounds the Lipschitz constant of the function $h$ (Assumption~\ref{cond-h}(i)).

\begin{lemma} \label{lem6}
For any given time interval $[0, T]$ and $\bar \epsilon > 0$, there exists an open neighborhood $D_{\bar \epsilon}(\lambda^*)$ of $\lambda^*$ such that for all $\lambda' \in D_{\bar \epsilon}(\lambda^*)$ and initial conditions $x \in \unitS$,
\begin{equation} \label{eq-l6}
  \| \phi_{c,\lambda'}(t; x) - \phi_\infty(t; x) \| \leq (\bar \epsilon + \epsilon(c) T) \, e^{\bar C LT}, \ \  \forall \, t \in [0, T],
\end{equation}  
where $\epsilon(c) > 0$ satisfies $\epsilon(c) \to 0$ as $c \to \infty$ and is independent of $x$, $\bar \epsilon$ and the choice of $\lambda^*$. 
\end{lemma}  

The proof of this lemma has many details, so we give its corollary first.

\begin{cor} \label{cor1}
There exist $\bar T > 0$, $\bar c \geq 1$,  both independent of the choice of $\lambda^*$, and an open neighborhood $D(\lambda^*)$ of $\lambda^*$ such that
for all $\lambda' \in D(\lambda^*)$ and initial conditions $x \in \unitS$, $\| \phi_{c, \lambda'}(t; x) \| < 1/4$ for all $t \in [\bar T, \bar T+1]$ and $c \geq \bar c$.
\end{cor}

\begin{proof}
 Let $\bar T$ be the time $T$ given by Lemma~\ref{lem5}, and let $T' = T + 1$. Choose $\bar \epsilon > 0$ small enough and $\bar c$ large enough so that for all $c \geq \bar c$, $(\bar \epsilon + \epsilon(c) T') \, e^{\bar C L T'} < 1/8$, where the function $\epsilon(c)$ is as in Lemma~\ref{lem6} for the interval $[0, T']$. Then let $D(\lambda^*)$ be given by Lemma~\ref{lem6} for the interval $[0, T']$ and the chosen $\bar \epsilon$.
\end{proof}

\begin{proof}[Proof of Lemma~\ref{lem6}]
First, we derive a few basic properties of $\phi_\infty$ that will be needed in the proof. 
Observe that $h_\infty$ is Lipschitz continuous with modulus $L$ and $h_\infty(0) = 0$ under Assumption~\ref{cond-h}.
For any $x, x' \in \unitS$, since
$\phi_\infty(t; x) - \phi_\infty(t; x')  = (x - x') + \int_0^t \lambda^*(s) [ h_\infty( \phi_\infty(s; x))  - h_\infty( \phi_\infty(s; x') ) ] ds,$
we have 
\begin{align*}
 & \left\| \phi_\infty(t; x) - \phi_\infty(t; x')  \right\|  \leq \| x - x' \| + \bar C L \textstyle{ \int_0^t  \left\| \phi_\infty(s; x)  -  \phi_\infty(s; x')  \right\| ds},
\end{align*}  
so by the Gronwall inequality \cite[Chap.\ 11, Lem.\ 6]{Bor09},
\begin{equation} \label{eq-l6-prf1}
  \! \left\| \phi_\infty(t; x) - \phi_\infty(t; x')  \right\| \leq \| x - x'\| e^{ \bar C L T}, \ \  \forall \, 
   t \in [0, T].
\end{equation}
A similar argument applied to $\phi_\infty(t; x) = x + \int_0^t \lambda^*(s) h_\infty( \phi_\infty(s; x)) ds$ shows that 
\begin{equation} \label{eq-l6-prf2}
  \sup_{ t \in [0, T], \, x: \| x\| = 1} \|\phi_\infty(t; x) \| \leq B \= e^{\bar C L T}  < \infty.
\end{equation} 
Then 
$\sup_{ t \in [0, T]} \sup_{x: \| x\| = 1} \| h_\infty( \phi_\infty(t; x)) \| \leq \bar B \= L B$ by the Lipschitz continuity of $h_\infty$. Thus for any $t \leq T$ and $x \in \unitS$, the function $h_\infty( \phi_\infty(\cdot; x))$ on $[0, t]$ belongs to $L_2([0, t]; \R^{d})$. Note also that the preceding bounds hold regardless of the choice of $\lambda^*$.

To construct the desired neighborhood $D_{\bar \epsilon}(\lambda^*)$, we choose a sufficiently small $\delta > 0$ such that $\delta ( 3 + 2 \bar C L T  e^{\bar C L T}) \leq \bar \epsilon$, and pick a finite number of points, $x^1, x^2, \ldots, x^\ell$, on the unit sphere $\unitS$ to form a $\delta$-cover of $\unitS$ (i.e., for any $x$ with $\| x \|=1$, $\| x - x^j\| \leq \delta$ for some $j \leq \ell$). We then choose an integer $m \geq \bar C \bar B/\delta$ and let $t_k = \tfrac{k T}{m}$, $0 \leq k \leq m$.
For each $1 \leq j \leq \ell$ and $1 \leq k \leq m$, define an open neighborhood of $\lambda^*$ in $\Upsilon$ by
$$ \textstyle{D_{jk} \= \left\{ \lambda' \in \Upsilon \Big| \left\| \int_0^{t_k} \big( \lambda'(s) - \lambda^*(s) \big) \, h_\infty( \phi_\infty(s; x^j)) \, ds   \right\| < \delta \right\}}.$$ 
Let us show that $D_{\bar \epsilon}(\lambda^*) \= \cap_{j=1}^\ell \cap_{k=1}^m D_{jk}$ (which is also an open neighborhood of $\lambda^*$ by definition) satisfies the desired relation (\ref{eq-l6}).

Consider an arbitrary $\lambda' \in D_{\bar \epsilon}(\lambda^*)$ and $c \geq 1$. For $t \leq T$ and $x \in \unitS$,
\begin{align}
    \phi_{c,\lambda'}(t; x) - \phi_\infty(t; x)  
  & = \textstyle{ \int_0^t \lambda'(s) h_c( \phi_{c,\lambda'}(s; x)) ds - \int_0^t \lambda^*(s) h_\infty( \phi_\infty(s; x)) ds}  \notag \\
   & = \textstyle{ \int_0^t \big( \lambda'(s) -  \lambda^*(s) \big) \, h_\infty( \phi_\infty(s; x)) \, ds}   \notag \\
    & \quad \, + \textstyle{ \int_0^t \lambda'(s) \, \big[ h_c( \phi_{c,\lambda'}(s; x)) - h_\infty( \phi_\infty(s; x)) \big] \, ds.} \label{eq-l6-prf3}
\end{align}
To bound $\| \phi_{c,\lambda'}(t; x) - \phi_\infty(t; x) \|$, we bound separately the two terms on the right-hand side (r.h.s.) of (\ref{eq-l6-prf3}).

To bound the first term, let $j \leq \ell$ be such that $\| x - x^j \| \leq \delta$ and let $k \leq m$ be the largest integer with $t_k \leq t$. 
Then 
$$  \textstyle{ \left \| \int_0^{t_k} \big( \lambda'(s) - \lambda^*(s) \big) \, h_\infty( \phi_\infty(s; x^j))  \, ds \right\| } \leq \delta,$$
since in the case $t_k > 0$, $\lambda' \in D_{jk}$ by the definition of $D_{\bar \epsilon}(\lambda^*)$, and in the case $t_k = 0$ the above inequality trivially holds.
Moreover, since $t - t_k \leq \tfrac{1}{m}$, for any $y \in \unitS$,
$$\textstyle{ \left \| \int_{t_k}^t \lambda'(s) h_\infty( \phi_\infty(s; y)) ds \right \|} \leq \bar C \Bar B/m \leq \delta, \qquad 
\textstyle{ \left \| \int_{t_k}^t \lambda^*(s) h_\infty( \phi_\infty(s; y)) ds \right \|} \leq \bar C \Bar B/m \leq \delta.$$
Using these relations with (\ref{eq-l6-prf1}) and the Lipschitz continuity of $h_\infty$, we obtain
\begin{align*}
 & \textstyle{ \left\| \int_0^t \big( \lambda'(s) - \lambda^*(s) \big) \, h_\infty( \phi_\infty(s; x)) \, ds  \right\|} \\
   &  \leq \textstyle{ \left\| \int_0^{t_k}  \big( \lambda'(s) -  \lambda^*(s) \big) \, h_\infty( \phi_\infty(s; x^j)) \, ds \right\|}   
   + \textstyle{ \left \| \int_{t_k}^t \lambda'(s) h_\infty( \phi_\infty(s; x^j)) ds \right \| + \left \| \int_{t_k}^t \lambda^*(s) h_\infty( \phi_\infty(s; x^j)) ds \right \|} \\
   & \quad \, +  \textstyle{ \left\| \int_0^t \lambda'(s) \big[ h_\infty( \phi_\infty(s; x)) - h_\infty( \phi_\infty(s; x^j)) \big] ds \right\|}  
   + \textstyle{ \left\| \int_0^t \lambda^*(s) \big[ h_\infty( \phi_\infty(s; x)) - h_\infty( \phi_\infty(s; x^j)) \big] ds \right\|} \\
   & \leq 3 \delta + 2 \bar C L T \cdot \delta e^{\bar C L T}.
\end{align*}
Hence by our choice of $\delta$, 
\begin{equation} \label{eq-l6-prf4}
  \textstyle{ \left\| \int_0^t \big( \lambda'(s) - \lambda^*(s) \big) \, h_\infty( \phi_\infty(s; x)) \, ds  \right\|} \leq  \bar \epsilon.
\end{equation}

To bound the second term on the r.h.s.\ of (\ref{eq-l6-prf3}), recall the constant $B$ defined in (\ref{eq-l6-prf2}), and let
$$ \epsilon(c) \= \bar C \sup_{ y : \| y\| \leq B} \| h_c (y) - h_\infty (y) \|.$$
Since $h_c(y)$ converges to $h_\infty(y)$ uniformly on compacts as $c \to \infty$ (Assumption~\ref{cond-h}), $\epsilon(c)$ satisfies that $\epsilon(c) \to 0$ as $c \to \infty$.
Now for $s \leq t$, by the Lipschitz continuity of $h_c$ (Assumption~\ref{cond-h}),
\begin{align*}
  \left\| \lambda'(s) \big[ h_c( \phi_{c,\lambda'}(s; x)) -  h_\infty( \phi_\infty(s; x)) \big] \right\|  
  & \leq \bar C \| h_c( \phi_{c,\lambda'}(s; x)) -  h_c( \phi_\infty(s; x)) \|  \\ & \quad \, 
  + \bar C \| h_c( \phi_\infty(s; x)) -  h_\infty( \phi_\infty(s; x)) \| \\
  & \leq \bar C L \|  \phi_{c,\lambda'}(s; x) -  \phi_\infty(s; x) \| + \epsilon(c).
\end{align*}
Therefore, 
\begin{align} 
& \textstyle{ \left\| \int_0^t \lambda'(s) \big[ h_c( \phi_{c,\lambda'}(s; x)) -  h_\infty( \phi_\infty(s; x)) \big] ds \right\|}  
\leq  \epsilon(c)T  + \bar C L \textstyle{ \int_0^t  \|  \phi_{c,\lambda'}(s; x) -  \phi_\infty(s; x) \| ds.} \label{eq-l6-prf5}
 \end{align} 
 
It follows from (\ref{eq-l6-prf3}), (\ref{eq-l6-prf4}), and (\ref{eq-l6-prf5}) that for $t \leq T$,
\begin{align*}
& \left\| \phi_{c,\lambda'}(t; x) - \phi_\infty(t; x) \right\|   \leq  \bar \epsilon + \epsilon(c) T  +  \bar C L  \textstyle{\int_0^t  \|  \phi_{c,\lambda'}(s; x) -  \phi_\infty(s; x) \| ds }.
\end{align*}
Then by the Gronwall inequality \cite[Chap.\ 11, Lem.\ 6]{Bor09},
$$ \left\| \phi_{c,\lambda'}(t; x) - \phi_\infty(t; x) \right\| \leq  (\bar \epsilon + \epsilon(c)T ) e^{\bar C L T}.$$
This proves that $D_{\bar \epsilon}(\lambda^*)$ satisfies (\ref{eq-l6}).
\end{proof}

We are almost ready to prove the boundedness of $\{x_n\}$ from algorithm (\ref{eq-alg0}). \emph{In what follows, let $\bar T > 0$ and $\bar c \geq 1$ be as given in Cor.~\ref{cor1}, and use $T \= \bar T + 1/2$ in defining the sequence of times, $T_n, n \geq 0$, for the processes $\hat x(\cdot)$ and the solutions $x^n(\cdot)$} introduced earlier in Section~\ref{sec-stab-prf1} [cf.\ \eqref{eq-def-tm}, \eqref{eq-hx0}, and \eqref{eq-ode0}]. 

\begin{lemma} \label{lem7}
Almost surely, there exists a sample path-dependent integer $\bar n \geq 0$ such that for all $n \geq \bar n$, $\lambda'_n \= \lambda(T_n + \cdot) \in \Upsilon$ satisfies that $\| \phi_{c, \lambda'_n}(t; x) \| < 1/4$ for all $t \in [\bar T, \bar T+1]$, $c \geq \bar c$, and initial conditions $x \in \unitS$.
\end{lemma}

\begin{proof}
Consider a sample path for which Lemma~\ref{lem4} holds. Let $G$ be the set of all limit points of $\{\lambda(T_n + \cdot)\}_{n \geq 0}$ in $\Upsilon$. Since $G$ is a closed subset of the compact metrizable space $\Upsilon$, $G$ is compact. By Lemma~\ref{lem4}, every $\lambda^* \in G$ is of the form $\lambda^*(t) = \rho(t) I$ with $\tfrac{1}{d} \leq \rho(t) \leq C$ for all $t \geq 0$.
Then $G$ is contained in the open set $D' \= \cup_{\lambda^* \in G} D(\lambda^*)$, where $D(\lambda^*)$ is the open neighborhood of $\lambda^*$ given by Cor.~\ref{cor1}. Since the sequence $\{\lambda(T_n + \cdot)\}_{n \geq 0}$ converges to $G$, it follows from the compactness of $G$ that for some finite integer $\bar n$, $\lambda'_n = \lambda(T_n + \cdot) \in D'$ for all $n \geq \bar n$. This means that if $n \geq \bar n$, $\lambda'_n \in D(\lambda^*)$ for some $\lambda^* \in G$, so we obtain the desired conclusion by Cor.~\ref{cor1}.
\end{proof}

Using the preceding results and the same reasoning as in the proof of \cite[Chap.\ 3, Thm.~7]{Bor09} for synchronous algorithms, we can now complete the proof of Theorem~\ref{thm-1}. We give the details below for clarity and completeness.

\begin{proof}[Proof of Theorem~\ref{thm-1}] 
The set of sample paths for which Lemmas~\ref{lem3} and~\ref{lem7} hold has probability $1$.
Consider any sample path from this set. 
We first prove that $\{\bar x (T_n) \}_{n \geq 0}$ must be bounded. 
Suppose this is not true. Then we can find a subsequence $T_{n_k}, k \geq 0$, with $\bar c \leq r(n_k) = \| \bar x (T_{n_k}) \| \uparrow \infty$. Let us derive a contradiction.

By Lemma~\ref{lem3}, for all $n$ sufficiently large, $\sup_{t \in [T_n, T_{n+1})} \| \hat x(t) - x^n(t)\| \leq 1/4$. 
This together with Lemma~\ref{lem7} implies that there exists some $\bar n'$ such that if $n \geq \bar n'$ and $r(n) \geq \bar c$, then $\|\hat x (T^-_{n + 1})\| < 1/2$ where, as defined earlier, $\hat x (T^-_{n + 1}) \= \lim_{t \uparrow T_{n + 1}} \hat x(t) = \bar x(T_{n+1})/r(n)$.
Since $r(n) = \| \bar x(T_n)\|$ in this case, we have
\begin{equation} \label{eq-t1-prf1}
 \frac{\| \bar x (T_{n + 1})\|}{ \| \bar x (T_{n})\|} = \| \hat x (T^-_{n + 1})\| < \frac{1}{2} \quad \text{if  $n \geq \bar n'$ and $r(n) \geq \bar c$}.
 \end{equation}
Now corresponding to each $n_k$, let $n'_k \= \max \{ n :   \bar n' \leq n < n_k, \,  r(n) < \bar c \}$ with $n'_k \= \bar n'$ if the set in this definition is empty.  
Then according to (\ref{eq-t1-prf1}), only two cases are possible: either 
\begin{enumerate}
\item[(i)] $n'_k = \bar n'$ and $r(\bar n') > 2 \, r(\bar n' + 1) > \cdots > 2^{n_k - \bar n'}  r(n_k) \geq 2^{n_k - \bar n'} \bar c$; or
\item[(ii)] $r(n'_k) < \bar c$ and $r(n'_k + 1) > 0.9 \, r(n_k)$. 
\end{enumerate}
Since $r(n_k) \uparrow \infty$, case (i) cannot happen for infinitely many $k$, and we must have case (ii) for all $k$ sufficiently large and with $n'_k \to \infty$ as $k \to \infty$.
Thus we have found infinitely many time intervals $[T_{n'_k}, T_{n'_k+1}]$ during each of which the process $\bar x$ starts from inside the ball of radius $\bar c$ and ends up outside a ball with an increasing radius $0.9 \, r(n_k) \uparrow \infty$. 

But this is impossible:
By Lemma~\ref{lem3}(i) for $\{\hat x(\cdot)\}$, $K^* \= \sup_{n \geq 0} \sup_{t \in [T_n, T_{n+1})} \| \hat x(t) \| < \infty$, 
and since $\hat x(t) = \bar x(t)/ r(n)$ on $[T_n, T_{n+1})$, 
this implies that if $r(n) < \bar c$, then $\|\bar x(t) \|< \bar c K^*$ for all $t \in [T_n, T_{n+1}]$.
This contradiction shows that $\{\bar x (T_n) \}_{n \geq 0}$ must be bounded.

Then $\{\bar x(t)\}_{t \geq 0}$ must also be bounded, because
\begin{align*}
 \sup_{n \geq 0} \sup_{t \in [T_n, T_{n+1})} \| \bar x(t) \|  & = \sup_{n \geq 0} \sup_{t \in [T_n, T_{n+1})}  r(n) \| \hat x(t) \|  \leq K^* \sup_{n \geq 0} r(n) < \infty.
\end{align*} 
This proves that $\{x_n\}$ is bounded.
\end{proof}

\section{Convergence Analysis} \label{sec-cvg}

We now move forward to prove Theorem~\ref{thm-2} regarding the convergence of the iterates $\{x_n\}$. 
While we could work with the continuous trajectory $\bar x(t)$ in our stability proof, it has the inconvenience of resulting in potentially multiple corresponding limiting ODEs. This situation could complicate our presentation by introducing unnecessary details when proving convergence.
Therefore, as described in Section~\ref{sec-prel-ana}, we opt to redefine $\bar x(t)$ with the random stepsizes $\tl \alpha_n = \sum_{i \in Y_n}  \alpha_{\nu(n, i)}$, $n \geq 0$, determining the elapsed time between consecutive iterates. In particular, for $n \geq 0$, let $\bar x(\tl t(n)) \= x_n$ and 
\begin{equation} 
 \bar x(t) \=  x_n +  \tfrac{t - \tl t(n)}{\tl t(n+1) - \tl t(n)} \, ( x_{n+1} - x_n), \ \ \,  t \in [\tl t(n), \tl t(n+1)]. \notag
\end{equation} 
where $\tl t(n)= \sum_{k=0}^{n -1} \tl \alpha_k$ as defined earlier in \eqref{def-conv-ana-ode-time}. We will refer to the temporal coordinate of $\bar x(t)$ as the `ODE-time.'
By the virtue of  Lemma~\ref{lem-cvg-2}, this redefinition of $\bar x(t)$ renders the corresponding limiting ODE unique and allows us to directly apply the available convergence results from \cite[Chap.~2]{Bor09}.

We will now give the main proof arguments for Theorem~\ref{thm-2}. Consider algorithm \eqref{eq-alg0} in its equivalent form \eqref{eq-alg1}; that is, in vector notation, 
\begin{equation} \label{eq-alg1a}
    x_{n+1}  = x_n + \tl \alpha_n  \tl \Lambda_n \left( h (x_n) + M_{n+1} + \epsilon_{n+1} \right),
\end{equation}
where $\tl \Lambda_n \=  \text{diag} \big( \tl \q(n, 1), \tl \q(n, 2), \ldots, \tl \q(n, d) \big)$, with diagonal entries $\tl \q(n, i) = \frac{\alpha_{\nu(n, i)}}{\tl \alpha_n} \ind\{i \in Y_n\}$ as defined previously. 
Note that by Assumptions~\ref{cond-ss}(i) and~\ref{cond-us}(i), $\{\tl \alpha_n \}$ satisfies 
\begin{equation} \label{eq-stepsize}
 \sum_n \tl \alpha_n = \infty, \quad \ \ \sum_n {\tl \alpha}^2_n < \infty, \ \ \text{a.s.}
\end{equation} 

\begin{lemma} \label{lem-cvg-3}
The sequence $\zeta_n \= \sum_{k=0}^{n-1} \tl \alpha_k \tl \Lambda_k M_{k+1}$, $n \geq 1$, converges a.s.\ in $\R^{d}$.
\end{lemma}

\begin{proof}
For integers $N \geq 1$, define stopping times $\tau_N$ and auxiliary variables $M^{(N)}_{k}$ as follows:
$$\tau_N \= \min \{ k \geq 0: \| x_k \| > N \}, \qquad M^{(N)}_{k+1} \= \ind \{ k < \tau_N \} M_{k+1},  \ \ \ k \geq 0.$$
By Assumption~\ref{cond-ns}(i), for each $N$, $\{M^{(N)}_k \}_{k \geq 1}$ is a martingale difference sequence with $\E[ \| M^{(N)}_{k+1} \|^2 \mid \F_k ] \leq K (1 + N^2 )$. Then the sequence $\{ \zeta^{(N)}_n\}_{n \geq 0}$ given by $\zeta^{(N)}_n \= \sum_{k=0}^{n-1} \tl \alpha_k \tl \Lambda_k M^{(N)}_{k+1}$ with $\zeta^{(N)}_0 \= 0$ is a square-integrable martingale (since the diagonal matrix $\tl \alpha_k \tl \Lambda_k$ has diagonal entries $\alpha_{\nu(k,i)} \ind \{ i \in Y_k\}, i \in \I$, all bounded by the finite constant $\sup_n \alpha_n$). Furthermore, since almost surely, 
$$\sum_{n=0}^\infty \E \left[ \| \zeta^{(N)}_{n+1} - \zeta^{(N)}_n \|^2 \!\mid\! \F_n \right] \leq \sum_{n=0}^\infty \tl \alpha^2_n \| \tl \Lambda_n \|^2 \, \E \left[ \| M^{(N)}_{n+1} \|^2 \!\mid \!\F_n \right] \leq \sum_{n=0}^\infty \tl \alpha_n^2 \| \tl \Lambda_n \|^2 K (1 + N^2) < \infty$$
(where the last inequality follows from (\ref{eq-stepsize}) and the fact that the entries of $\tl \Lambda_n$ lie in $[0,1]$), we have that $\{ \zeta^{(N)}_n\}_{n \geq 0}$ converges a.s.\ in $\R^{d}$ by \cite[Prop.\ VII-2-3(c)]{Nev75}.
As $\{x_n\}$ is bounded a.s.\ by Theorem~\ref{thm-1}, the definitions of $\tau_N$ and $\{M^{(N)}_k\}$ imply that almost surely, $\{\zeta_n\}_{n \geq 1}$ coincides with $\{\zeta^{(N)}_n\}_{n \geq 1}$ for some sample path-dependent value of $N$, leading to the a.s.\ convergence of $\{\zeta_n\}_{n \geq 1}$ in $\R^{d}$.
 \end{proof}

The next step in the proof involves using Lemma~\ref{lem-cvg-3} and Theorem~\ref{thm-1} to show that the trajectory $\bar x(\cdot)$ asymptotically `tracks' the solutions of two ODEs. The first ODE is defined by the random trajectory $\tl \lambda(\cdot)$, while the second one is the limiting ODE obtained using Lemma~\ref{lem-cvg-2}:
\begin{align}
  \dot{x}(t) & = \tl \lambda(t) h (x(t)),  \label{eq-ode1} \\ 
   \dot{x}(t) & = \tfrac{1}{d} h(x(t)).  \label{eq-ode2}
\end{align}   
Let $T > 0$. For $s \geq 0$, let ${\tl x}^s(\cdot)$ and $x^s(\cdot)$ be the unique solutions of (\ref{eq-ode1}) and (\ref{eq-ode2}), respectively, on the time interval $[s, s+T]$ with initial conditions ${\tl x}^s(s) = x^s(s) = \bar x(s)$. For $s \geq T$, let ${\tl x}_s(\cdot)$ and $x_s(\cdot)$ be the unique solutions of (\ref{eq-ode1}) and (\ref{eq-ode2}), respectively, on the time interval $[s-T, s]$ with terminal conditions ${\tl x}_s(s) = x_s(s) = \bar x(s)$.

\begin{lemma} \label{lem-cvg-4}
For any $T > 0$, almost surely,
\begin{align} 
 \lim_{s \to \infty} \sup_{t \in [s, s+T]} \| \bar x(t) - \tl x^s(t)\|  & = 0,  \label{eq-lc4-1a} \\
   \lim_{s \to \infty} \sup_{t \in [s-T, s]} \| \bar x(t) - \tl x_s(t)\|  & = 0, \label{eq-lc4-1b} \\ 
   \lim_{s \to \infty} \sup_{t \in [s, s+T]} \| \bar x(t) - x^s(t)\| & = 0, \qquad \label{eq-lc4-2a} \\
   \lim_{s \to \infty} \sup_{t \in [s-T, s]} \| \bar x(t) - x_s(t)\|  & = 0. \label{eq-lc4-2b}
\end{align}
\end{lemma}
  
\begin{proof}
Consider a sample path for which Theorem~\ref{thm-1}, Lemmas~\ref{lem-cvg-3} and \ref{lem-cvg-2}, and all the assumptions hold. 
To prove (\ref{eq-lc4-1a})-(\ref{eq-lc4-1b}), we work with (\ref{eq-alg1a}): 
$$x_{n+1}  = x_n + \tl \alpha_n  \tl \Lambda_n \! \left( h (x_n) + M_{n+1} + \epsilon_{n+1} \right),$$
and observe the following: 
\begin{enumerate}
\item[(i)] $\{x_n\}$ is bounded by Theorem~\ref{thm-1}; 
\item[(ii)] $\sum_n \tl \alpha_n = \infty$, $\sum_n {\tl \alpha}^2_n < \infty$ by (\ref{eq-stepsize});
\item[(iii)] $\|\tl \Lambda_n \|, n \geq 0$, and $\tl \lambda(t), t \geq 0$ are bounded by deterministic constants by definition; 
\item[(iv)] $h$ is Lipschitz continuous by Assumption~\ref{cond-h}(i); 
\item[(v)] as $n \to \infty$,  $\sup_{m \geq 0} \left\| \sum_{k=n}^{n+m} \tl \alpha_k \tl \Lambda_k M_{k+1} \right\| \to 0$ by Lemma~\ref{lem-cvg-3}; and $\epsilon_n \to 0$ by Assumption~\ref{cond-ns}(ii) and Theorem~\ref{thm-1}.
\end{enumerate}
Using the above observations, we can essentially replicate the proof of \cite[Chap.\ 2, Lem.~1]{Bor09} step by step, with some minor variations, to obtain (\ref{eq-lc4-1a})-(\ref{eq-lc4-1b}).

To prove \eqref{eq-lc4-2a}-\eqref{eq-lc4-2b}, we first establish their validity when we substitute $\bar x$ in these relations with $\tl x^s$ and $\tl x_s$, respectively. This proof involves Theorem~\ref{thm-1}, Lemma~\ref{lem-cvg-2}, and an application of Borkar \cite[Lem.\ 3.1(ii)]{Bor98}, which deals with solutions of ODEs of the form \eqref{eq-ode1} and their simultaneous continuity in both the $\tl \lambda$ function and the initial condition. Combining this result with the relations \eqref{eq-lc4-1a}-\eqref{eq-lc4-1b} then leads to \eqref{eq-lc4-2a}-\eqref{eq-lc4-2b}. We now give the details.

Let $\C([0, T]; \R^{d})$ denote the space of all $\R^{d}$-valued continuous functions $f$ on $[0, T]$ with the sup-norm
$\| f\| \= \sup_{t \in [0, T]} \| f(t)\|$. 
Let $\Psi_1$ (respectively, $\Psi_2$) denote the mapping that maps each $(\lambda', x^o) \in \tl \Upsilon \times \R^{d}$ to the unique solution of the ODE $\dot{x}(t) = \lambda'(t) h (x(t)), t \in [0,T]$, with the initial condition $x(0) = x^o$ (respectively, the terminal condition $x(T) = x^o$). Since $h$ is Lipschitz continuous, by Borkar \cite[Lem.\ 3.1(ii)]{Bor98}, $\Psi_1$ and $\Psi_2$ are continuous mappings from $\tl \Upsilon \times \R^{d}$ into the space $\C([0, T]; \R^{d})$. Therefore, if we equip the space $\tl \Upsilon$ with a metric consistent with its topology, then $\Psi_1$ and $\Psi_2$ are uniformly continuous on any compact subset of $\tl \Upsilon \times \R^{d}$, in particular, on the compact set $\tl \Upsilon \times \overline{\{\bar x(t) : t \geq 0\}}$, where $\overline{\{\bar x(t) : t \geq 0\}}$ denotes the closure of the set $\{\bar x(t) : t \geq 0\}$ and is compact by Theorem~\ref{thm-1}.
Consequently, since $\tl \lambda (t + \cdot) \to \bar \lambda(\cdot) \equiv \tfrac{1}{d} I$ as $t \to \infty$ (Lemma~\ref{lem-cvg-2}) and the initial (respectively, terminal) conditions $\tl x^s(s) = x^s(s)$ (respectively, $\tl x_s(s)=x_s(s)$) all lie in $\{\bar x(t) : t \geq 0\}$, we obtain that
 $\lim_{s \to \infty} \sup_{t \in [s, s+T]} \| \tl x^s(t) - x^s(t) \|  = 0$ and 
 $\lim_{s \to \infty} \sup_{t \in [s-T, s]} \| \tl x_s(t) - x_s(t)\|  = 0.$
Together with (\ref{eq-lc4-1a})-(\ref{eq-lc4-1b}) proved earlier, this implies (\ref{eq-lc4-2a})-(\ref{eq-lc4-2b}).
\end{proof}

We now prove the convergence results in Theorem~\ref{thm-2}. 

\begin{proof}[Proof of Theorem~\ref{thm-2}]
Using the a.s.\ boundedness of $\{x_n\}$ given by Theorem~\ref{thm-1} and the relations (\ref{eq-lc4-2a}) and (\ref{eq-lc4-2b}) given by Lemma~\ref{lem-cvg-4}, the same proof of \cite[Chap.\ 2, Thm.~2]{Bor09} goes through here and establishes that $\{x_n\}$ converges a.s.\ to a, possibly sample path-dependent, compact connected internally chain transitive invariant set of the ODE $\dot{x}(t) = \tfrac{1}{d} h (x(t))$. The solutions of this ODE are simply the solutions of the ODE $\dot{x}(t) = h (x(t))$ by a constant time scaling, so the two ODEs have identical compact connected internally chain transitive invariant sets. The desired conclusion then follows. 
\end{proof}

Finally, we discuss another important consequence of Lemma~\ref{lem-cvg-4} and Theorems~\ref{thm-1} and \ref{thm-2}. It concerns the asymptotic behavior of $\bar x(\cdot)$ and has implications for RL applications. The proof is standard but will be provided for the sake of completeness.

Let us extend $\bar x(\cdot)$ from $[0, \infty)$ to $(-\infty, \infty)$ by setting $\bar x(\cdot) \equiv x_0$ on $(-\infty, 0)$, so that we can view $\bar x(\cdot)$ as a function in $\C ((-\infty, \infty); \R^{d})$. Here $\C ((-\infty, \infty); \R^{d})$ is the space of all $\R^{d}$-valued continuous functions on $(-\infty, \infty)$ equipped with a metric that renders the convergence of $f_n \to f$ in this space to mean uniform convergence of $f_n$ to $f$ on compact intervals (e.g., $d(f, g) \= \sum_{n=1}^\infty 2^{-n}  ( 1 \wedge \sup_{t \in [-n, n]} \| f(t) - g(t) \|)$ is such a metric). The space $\C ((-\infty, \infty); \R^{d})$ is complete and by the Arzel\'{a}-Ascoli theorem, a set $B \subset \C ((-\infty, \infty); \R^{d})$ is relatively compact (i.e., has compact closure) if and only if the collection of functions in $B$ is equicontinuous and pointwise bounded (cf.\ \cite[Chap.\ 11.1.1]{Bor09} or \cite[Chap.\ 4.2.1]{KuY03}).

\begin{theorem} \label{thm-3}
Consider the continuous trajectory $\bar x(\cdot)$ defined above. Under Assumptions~\ref{cond-h}--\ref{cond-us}, almost surely, the set $\{\bar x(t + \cdot)\}_{t \in \R}$ is relatively compact in $\C((-\infty, \infty); \R^{d})$, and any limit point of $\bar x(t + \cdot)$ as $t \to \infty$ is a solution of the ODE 
$\dot{x}(t) = \tfrac{1}{d} h(x(t))$ that lies entirely in some compact invariant set of this ODE.
\end{theorem}

\begin{proof}
Consider a sample path for which \eqref{eq-stepsize} and Theorems~\ref{thm-1} and~\ref{thm-2} hold and Lemma~\ref{lem-cvg-4} holds for all $T = 1 , 2, \ldots$.
By Theorem~\ref{thm-1}, $\{\bar x(t + \cdot)\}_{t \in \R}$ is uniformly bounded. Since $h$ is Lipschitz continuous, applying Gronwall's inequality \cite[Chap.\ 11, Lem.\ 6]{Bor09} shows that given a bounded set of initial conditions $x(0)$, the solutions of the ODE $\dot{x}(t) = \tfrac{1}{d} h(x(t))$ are equicontinuous on $(-\infty, \infty)$. Combining these two facts with the fact that (\ref{eq-lc4-2a}) and (\ref{eq-lc4-2b}) hold for all $T = 1 , 2, \ldots$, it follows that $\{\bar x(t + \cdot)\}_{t \in \R}$ is equicontinuous. Therefore, given its uniform boundedness, it is relatively compact in $\C \big((-\infty, \infty); \R^{d} \big)$.

Now let $x^*(\cdot) \in  \C \left((-\infty, \infty); \R^{d} \right)$ be the limit of any convergent sequence $\{\bar x(t_k+ \cdot)\}_{k \geq 1}$ with $t_k \to \infty$. 
Then $\bar x(t_k) \to x^*(0)$ and $\bar x(t_k + \cdot) \to x^*(\cdot)$ uniformly on each interval $[-T, T]$, $T = 1, 2, \ldots$, as $k \to \infty$. 
With (\ref{eq-lc4-2a})-(\ref{eq-lc4-2b}) holding for all these $T$, this implies $x^k(\cdot) \to x^*(\cdot)$ in $\C \left((-\infty, \infty); \R^{d} \right)$, where $x^k(\cdot)$ is the solution of the ODE $\dot{x}(t) = \tfrac{1}{d} h(x(t))$ on $(-\infty, \infty)$ with $x^k(0) = \bar x(t_k)$. On the other hand, since $x^k(0) \to x^*(0)$, by the Lipschitz continuity of $h$, $x^k(\cdot)$ also converges, uniformly on each compact interval, to the solution of the ODE $\dot{x}(t) = \tfrac{1}{d} h(x(t))$ with condition $x(0) = x^*(0)$. Therefore, $x^*(\cdot)$ must coincide with this solution.
From Theorem~\ref{thm-2} on the convergence of $\{x_n\}$ and the equicontinuity of $\{ \bar x(t_k + \cdot)\}_{k \geq 1}$ proved earlier, it follows that $\bar x(t_k)$ converges to some compact invariant set $D$ of the ODE $\dot{x}(t) = \tfrac{1}{d} h(x(t))$. Hence $x^*(0) \in D$. Since $D$ is invariant, this implies $x^*(t) \in D$ for all $t \in \R$.
\end{proof}

The following corollary specializes the preceding convergence results in Theorems~\ref{thm-2} and~\ref{thm-3} to a scenario relevant to average-reward RL applications. 

Let $E_h \= \{ x \in \R^d \mid h(x) = 0 \}$. It is worth noting that under Assumption~\ref{cond-h}, $E_h$ must be compact. In the context of our recent work on average-reward Q-learning for weakly communicating MDPs/SMDPs \cite{WYS24}, $E_h$ corresponds to a nonempty compact subset of solutions to the average-reward optimality equation and, in general, is not a singleton.  

\begin{cor} \label{cor-ql}
Suppose that Assumptions~\ref{cond-h}-\ref{cond-us} hold and that $E_h$ contains all compact invariant sets of the ODE $\dot{x}(t) = h(x(t))$. Then the following hold almost surely for the iterates $\{x_n\}$ generated by algorithm \eqref{eq-alg0}:\\ 
{\rm (i)} $\{x_n\}$ converges to the compact set $E_h$. \\
{\rm (ii)} For any $\delta > 0$ and any convergent subsequence $\{x_{n_k}\}$, as $k \to \infty$,
$$\tau_{\delta,k} \= \min \left\{ |s| :  \| \bar x(t_{n_k} + s) - x^* \| > \delta, \ s \in \R \right\} \to \infty,$$ 
where $\bar x(\cdot)$ is the continuous trajectory defined above, 
$t_{n_k} = \tl t(n_k)$ is the `ODE-time' when $x_{n_k}$ is generated, and $x^* \in E_h$ is the point to which $\{x_{n_k}\}$ converges.
\end{cor}
\begin{proof}
Under our assumptions, part (i) is implied by Theorem~\ref{thm-2}. For part (ii), by the definition of $E_h$, if $x(\cdot)$ is a solution of the ODE $\dot{x}(t) = \tfrac{1}{d} h(x(t))$ that lies entirely in $E_h$, then $x(\cdot) \equiv x^*$ for some $x^* \in E_h$.
Therefore, by Theorem~\ref{thm-3} and its proof, if $x_{n_k} \to x^* \in E_h$, then $\bar x(t_{n_k} + \cdot)$ converges to the constant function $x(\cdot) \equiv x^*$ in $\C \left((-\infty, \infty); \R^{d} \right)$ as $k \to \infty$. This means that $\bar x(t_{n_k} + s)$ converges to $x^*$ uniformly in $s$ on compact intervals. Consequently, we must have $\tau_{\delta,k} \to \infty$ as $k \to \infty$.  
\end{proof}

Corollary~\ref{cor-ql}(ii) shows that over time, algorithm \eqref{eq-alg0} will spend increasingly more `ODE-time' in arbitrarily small neighborhoods around its iterates' limit points, and the duration spent around each limit point tends to infinity. This suggests that while the sequence $\{x_n\}$ may not converge to a single point, its behavior can give the appearance of convergence. 

\section{Discussion}  \label{sec-conc-rmks}

In this paper, we have established the stability and convergence of a family of asynchronous SA algorithms that have important average-reward RL applications. Our stability analysis extends Borkar and Meyn's method to address more general noise conditions than previously considered in that framework. While we have focused on partially asynchronous schemes needed for average-reward RL, the ideas in our stability analysis, especially constructing an auxiliary scaled process with stopping techniques, could potentially apply to a broader range of asynchronous schemes, including those discussed in \cite{Bor98}, given suitable functions $h$.

Additionally, an alternative stability proof is available when the martingale-difference noises $\{M_n\}_{n \geq 1}$ adhere to the specific form assumed in the prior works \cite{Bor98,BoM00,ABB01}: $M_n = F(x_{n-1}, \zeta_n)$, involving i.i.d.\ exogenous variables $\{\zeta_n\}_{n \geq 1}$ and a function $F$ that is uniformly Lipschitz in its first argument, as discussed in Remark~\ref{rmk-cond}b. In this case, a slightly simpler stability proof can be derived by working with the continuous trajectory $\bar x(\cdot)$ and the $\tl \lambda$ function defined in the second part of Section~\ref{sec-prel-ana}. As the referenced works \cite{Bor98,BoM00,ABB01} did not explicitly provide a proof of this stability result, we include our alternative proof in the \hyperref[app-alt-stab]{Appendix}.

As a final remark, our analysis has focused on algorithms without communication delays, where each iteration uses the current iterate $x_n$ to evaluate the values of $h_i, i \in \I$. This, however, precludes distributed implementation scenarios where 
communication delays between processors may necessitate using past iterates for updating each component. Bhatnagar's work \cite{Bha11} has provided a stability proof for such distributed algorithms but under a much stronger noise condition compared to ours. A future work is to extend our stability analysis to a distributed computation framework that accounts for communication delays.

\appendix
\counterwithin{equation}{section}
\renewcommand{\theequation}{A.\arabic{equation}}
\counterwithin{mylemma}{section}
\renewcommand{\themylemma}{A.\arabic{mylemma}}
\counterwithin{myassumption}{section}
\renewcommand{\themyassumption}{A.\arabic{myassumption}}

\titleformat{\section}{\normalfont\Large\bfseries}{\appendixname:}{1em}{}
\renewcommand{\thesection}{}

\section{Alternative Stability Proof under a Stronger Noise Condition} \label{app-alt-stab}

In this appendix, we consider a stronger condition from Borkar \cite{Bor98} on the martingale difference noise sequence $\{M_{n}\}$, and give an alternative, simpler proof of the stability theorem for this case.

\begin{myassumption}[Alternative condition on $\{M_n\}$] \label{cond-alt-ns} \hfill \\
For all $n \geq 0$, $M_{n+1}$ is given by $M_{n+1} = F (x_{n}, \zeta_{n+1})$, where:
\begin{enumerate}
\item[\rm (i)] $\zeta_1, \zeta_2, \ldots$ are exogenous, i.i.d.\ random variables taking values in a measurable space $\Z$, with a common distribution $p$. 
\item[\rm (ii)] The function $F : \R^d \times \Z \to \R^d$ has these properties: It is uniformly Lipschitz continuous in its first argument; i.e., for some constant $L_F > 0$,
\begin{equation}   \notag 
\| F( x, z) - F(y, z) \| \leq L_F \| x - y \|, \quad  \forall \, x, y \in \R^d, \   z \in \Z.
\end{equation} 
It is measurable in its second argument and moreover,
$$  \int_\Z \| F(0, z) \|^2 \, p(dz) < \infty, \qquad \int_\Z  F(x, z) \, p(dz) = 0,  \ \ \ \forall \, x \in \R^d.$$ 
\end{enumerate}
\end{myassumption}

Assumption~\ref{cond-alt-ns} implies Assumption~\ref{cond-ns}(i). Indeed, using the properties of the function $F$, a direct calculation shows that for some constant $K_F > 0$,
\begin{equation} \label{eq-alt-prf0}
       \int _\Z \| F(x, z) \|^2 \, p(dz) \leq K_F \!\left( 1 + \| x\|^2 \right), \quad \forall \, x \in  \R^d.
\end{equation} 
Thus, with $\F_n \= \sigma(x_m, Y_m, \zeta_m, \epsilon_m; m \leq n)$, $\{M_{n+1}\}$ satisfies Assumption~\ref{cond-ns}(i) with
\begin{equation} \label{eq-alt-prf0b}
       \E [ \| M_{n+1} \|^2 \mid \F_n ] \leq K_F \!\left( 1 + \| x_n\|^2 \right), \quad n \geq 0.
\end{equation}

By leveraging the specific form of $\{M_{n+1}\}$, we simplify the proof of the stability theorem. In this case, unlike the previous analysis in Section~\ref{sec-stab}, we work with the linearly interpolated trajectory $\bar x(t)$ that was used in our convergence analysis (Section~\ref{sec-cvg}). Recall that, in defining it, we place the iterate $x_n$ at the `ODE-time' $\tl t(n) = \sum_{k=0}^{n-1} \tl \alpha_k$, with the random stepsize $\tl \alpha_k = \sum_{i \in Y_k}  \alpha_{\nu(k, i)}$ representing the elapsed time between the $k$th and $k+1$th iterates. In the same manner as before, we divide the time axis into intervals of approximately length $T$ for a given $T > 0$, and we define the scaled trajectory $\hat x(t)$ accordingly. In particular, $T_n$ and $m(n)$ are recursively defined by (\ref{eq-def-tm}), but with $\tl t(m)$ replacing $t(m)$: 
$$ m(0)= T_0 = 0 \ \ \ \text{and} \ \ \   m(n+1) \= \min \{ m : \tl t(m) \geq T_n + T \}, \ \  T_{n+1} \= \tl t\big(m(n+1)\big), \ \ n \geq 0.$$
Observe that \emph{$T_n$ and $m(n)$ are now random variables.} With $r(n) \= \| x_{m(n)} \| \vee 1$, we then have the scaled trajectory $\hat x(t)$ and a `copy' of it, $\hat x^n(t)$, on each closed internal $[T_n, T_{n+1}]$ given by the definitions presented in and below (\ref{eq-hx0}), as in the previous analysis.  

As we discussed in Section~\ref{sec-stab-prf1}, a key step in the stability analysis is to establish $\sup_t \| \hat x(t)\| < \infty$ a.s. We will now proceed to prove this.

For $m(n) \leq k < m(n+1)$, we can express $\hat x^n(\tl t(k+1))$ as
\begin{equation}  \label{eq-alt-hx}
  \hat x^n( \tl t(k+1)) = \hat x(\tl t(k)) + \tl \alpha_k \tl \Lambda_k h_{r(n)}(\hat x^n( \tl t(k))) +  \tl \alpha_k \tl \Lambda_k \hat M_{k+1} + \tl \alpha_k \tl \Lambda_k \hat \epsilon_{k+1},
\end{equation}  
where $\tl \Lambda_k$ is the diagonal matrix defined below (\ref{eq-alg1a}):
$$\tl \Lambda_k =  \text{diag} \big( \tl \q(k, 1), \tl \q(k, 2), \ldots, \tl \q(k, d) \big), \quad \text{with} \ \ \tl \q(k, i) = \alpha_{\nu(k, i)} \ind\{i \in Y_k\} / \tl \alpha_k,$$
$\hat M_{k+1} \= M_{k+1} / r(n) = F(x_k, \zeta_{k+1})/r(n)$ by Assumption~\ref{cond-alt-ns}, and $\hat \epsilon_{k+1} \= \epsilon_{k+1}/r(n)$. 
 
Let us introduce another noise sequence $\{\hat M^o_k\}$ related to $\{\hat M_{k}\}$. For $n \geq 0$ and $k \geq 0$, let
\begin{equation} \label{eq-alt-prf1}
    \hat M^o_{k+1}  \= F (0, \zeta_{k+1})/r(n) \ \ \text{if} \ m(n) \leq k < m(n+1).
\end{equation}  
Equivalently, by the definition of $m(n)$, for each $k \geq 0$, 
\begin{equation} \notag
 \hat M^o_{k+1} = F (0, \zeta_{k+1}) / r(\ell(k)), \ \  \ \text{where} \  \ \ell(k) \= \max \{ \ell \geq 0: T_\ell \leq \tl t(k) \}.
\end{equation} 
Observe that $r(\ell(k)) = \| x_{m(\ell(k))} \| \vee 1$ is $\F_k$-measurable. 
Therefore, by Assumption~\ref{cond-alt-ns}(ii) and (\ref{eq-alt-prf0}),
\begin{equation} \label{eq-alt-prf2}
  \E [ \hat M^o_{k+1} \mid \F_k] = 0, \qquad \E [ \| \hat M^o_{k+1} \|^2 \mid \F_k ] \leq K_F, \quad \forall \, k \geq 0.
\end{equation}
Moreover, by the Lipschitz continuity property of $F$, for $m(n) \leq k < m(n+1)$,
\begin{equation} \label{eq-alt-prf3}
   \| \hat M_{k+1} - \hat M^o_{k+1} \| = \frac{\| F(x_k, \zeta_{k+1}) - F(0, \zeta_{k+1}) \|}{r(n)}  \leq  \frac{L_F \| x_k\|}{r(n)} = L_F \| \hat x^n(\tl t(k)) \|.
\end{equation} 
 
\begin{mylemma} \label{lem-alt1}
Almost surely, the sequence $\zeta_n^o \= \sum_{k=0}^{n-1} \tl \alpha_k \tl \Lambda_k \hat M^o_{k+1}$ (with $\zeta_0^o = 0$) converges in $\R^d$.
\end{mylemma}

\begin{proof}
Since, for all $k \geq 0$, the stepsizes $\tl \alpha_k$ and the entries of $\tl \Lambda_k$ are bounded by deterministic constants, it follows from (\ref{eq-alt-prf2}) that $(\zeta_n^o, \F_n)$ is a square-integrable martingale and moreover, 
$$\sum_{n=0}^\infty \E \left[ \| \zeta^o_{n+1} - \zeta^o_n \|^2 \mid \F_n \right] \leq 
\sum_{n=0}^\infty \tl \alpha^2_n \| \tl \Lambda_n \|^2 \, \E \left[ \| \hat M^o_{n+1} \|^2 \mid \F_n \right] 
\leq K_F \sum_{n=0}^\infty \tl \alpha_n^2 \| \tl \Lambda_n \|^2 < \infty, \ \ \ a.s.$$
(since $\sum_n \tl \alpha_n^2 < \infty$ a.s.).
Then by \cite[Prop.\ VII-2-3(c)]{Nev75}, almost surely, $\zeta_n^o$ converges in $\R^d$.
\end{proof}

\begin{mylemma} \label{lem-alt2}
$\sup_{n \geq 0} \sup_{t \in [T_n, T_{n+1}]} \| \hat x^n(t) \|  < \infty$ a.s.
\end{mylemma}

\begin{proof}
As in \cite[Chap.\ 3, Lem.\ 6]{Bor09}, we will show that $\sup_{t \in [T_n, T_{n+1}]} \| \hat x^n(t) \|$ can be bounded by a number independent of $n$. For each $n \geq 0$, using (\ref{eq-alt-hx}), we have that for $k$ with $m(n) \leq k < m(n+1)$,
$$ \hat x^n(\tl t(k+1))  =   \hat x^n(\tl t(m(n)))  + \sum_{i=m(n)}^{k} \tl \alpha_i \tl \Lambda_i h_{r(n)}(\hat x^n(\tl t(i))) + 
\sum_{i=m(n)}^{k} \tl \alpha_i \tl \Lambda_i \hat M_{i+1} + \sum_{i=m(n)}^{k} \tl \alpha_i \tl \Lambda_i \hat \epsilon_{i+1}.$$
Similarly to the proof of \cite[Chap.\ 3, Lem.\ 6]{Bor09}, we proceed to bound $\| \hat x^n(\tl t(k+1))\|$ by bounding the norm of each term on the r.h.s.\ of the above equation. By the definition of $\hat x(\cdot)$, we have $\| \hat x^n(\tl t(m(n)) \| \leq 1$. Using the Lipschitz continuity of $h_c$ (Assumption~\ref{cond-h}) and the fact that $\sup_{i \geq 0} \| \tl \Lambda_i\|\leq \tl C$ for some deterministic constant $\tl C$, we can bound the norm of the second term by $\sum_{i=m(n)}^{k} \tl \alpha_i \tl C ( \| h(0)\| + L \| \hat x^n(\tl t(i))\| )$. For the forth term, by Assumption~\ref{cond-ns}(ii), we have $\|\hat \epsilon_{i+1} \| = \|  \epsilon_{i+1} \|/r(n) \leq \delta_{i+1} (1 + \| \hat x^n(\tl t(i))\| )$, so $\left\| \sum_{i=m(n)}^{k} \tl \alpha_i \tl \Lambda_i \hat \epsilon_{i+1} \right\| \leq \sum_{i=m(n)}^{k} \tl \alpha_i \tl C B_\delta (1 + \| \hat x^n(\tl t(i))\|)$, where $B_\delta \= \sup_{i \geq 1} \delta_i < \infty$ a.s..  
For the third term, we use (\ref{eq-alt-prf3}) and Lemma~\ref{lem-alt1} to obtain
\begin{align*}
    \left\| \sum_{i=m(n)}^{k} \tl \alpha_i \tl \Lambda_i \hat M_{i+1}  \right\| & \leq 
     \left\| \sum_{i=m(n)}^{k} \tl \alpha_i \tl \Lambda_i \hat M^o_{i+1} \right\| +  
       \sum_{i=m(n)}^{k} \tl \alpha_i  \| \tl \Lambda_i \|  \left\| \hat M_{i+1} - \hat M^o_{i+1} \right\|  \\
       & \leq \| \zeta^o_{k+1} - \zeta^o_{m(n)} \| +  \sum_{i=m(n)}^{k} \tl \alpha_i  \tl C  \cdot L_F\| \hat x^n(\tl t(i)) \| \\
       & \leq 2 B + L_F \tl C \!\sum_{i=m(n)}^{k} \tl \alpha_i  \| \hat x^n(\tl t(i)) \|,
 \end{align*}  
 where $B \= \sup_i \| \zeta^o_i\| < \infty$ a.s.\ (Lemma~\ref{lem-alt1}). Observe also that by the definitions of $m(n)$, $m(n+1)$, and $\tl \alpha_i$, we have
$\sum_{i=m(n)}^{m(n+1) -1} \tl \alpha_i < T + \tl \alpha_{m(n+1) -1} \leq T + d \bar \alpha$,  where $\bar \alpha \= \sup_j \alpha_j < \infty$.
By combining the preceding derivations, we obtain
 $$
 \| \hat x^n(\tl t(k+1)) \| \leq 1 +  2 B +  \tl C (T+d \bar \alpha) (\| h(0) \| + B_\delta) + \tl C(L + L_F + B_\delta) \!\sum_{i=m(n)}^{k} \tl \alpha_i \| \hat x^n(\tl t(i))\|.
 $$
 Then by the discrete Gronwall inequality \cite[Chap.\ 11, Lem.\ 8]{Bor09}, for all $k$ with $m(n) \leq k < m(n+1)$,
 $$  \| \hat x^n(\tl t(k+1)) \| \leq \left(1 +  2 B +  \tl C (T+d \bar \alpha) (\| h(0) \| + B_\delta) \right) e^{\tl C (L+ L_F + B_\delta) (T+d \bar \alpha)}.$$
This shows that almost surely, $\sup_{t \in [T_n, T_{n+1}]} \| \hat x(t) \|$ can be bounded by a finite (random) number independent of $n$, and therefore, $\sup_{n \geq 0} \sup_{t \in [T_n, T_{n+1}]} \| \hat x(t) \| < \infty$ a.s.
\end{proof}
 
With Lemma~\ref{lem-alt2}, we have established the boundedness of the scaled trajectory $\hat x(\cdot)$. This has the following implication, which will be needed shortly in relating $\{\hat x^n(\cdot)\}$ to ODE solutions:

\begin{mylemma} \label{lem-alt3}
Almost surely, as $n \to \infty$, $\hat \epsilon_n \to 0$, and $\zeta_n \= \sum_{k=0}^{n-1} \tl \alpha_k \tl \Lambda_k \hat M_{k+1}$ converges in $\R^d$.
\end{mylemma}

\begin{proof}
By the definition of $\{\hat \epsilon_k\}$ and Assumption~\ref{cond-ns}(ii), we have that for $m(n) \leq k < m(n+1)$, $\|\hat \epsilon_{k+1} \| = \|  \epsilon_{k+1} \|/r(n) \leq \delta_{k+1} (1 + \| \hat x^n(\tl t(k))\|)$, where $\delta_{k} \to 0$ a.s., as $k \to \infty$. 
By Lemma~\ref{lem-alt2}, this implies $\hat \epsilon_k \to 0$ a.s., as $k \to \infty$.

The proof of the a.s.\ convergence of $\{\zeta_n\}$ is similar to the proof of Lemma~\ref{lem-cvg-3} in Section~\ref{sec-cvg}. Specifically, for integers $N \geq 1$, we define stopping times $\tau_N$ and auxiliary variables ${\hat M}^{(N)}_k$ by
\begin{align*}
  \tau_N  & \= \, \min \left\{ k \geq 0 \, \big|\,  \| \hat x^n(\tl t(k))\| > N, \, m(n) \leq k < m(n+1), \, n \geq 0 \right\},  \\
 {\hat M}^{(N)}_{k+1} & \= \, \ind \{ k < \tau_N \} \hat M_{k+1},  \quad k \geq 0.
\end{align*} 
Using Assumption~\ref{cond-alt-ns}, \eqref{eq-alt-prf0b}, and the definition of $\hat M_{k+1}$, we have that for each $N$,  $\{{\hat M}^{(N)}_k \}_{k \geq 1}$ is a martingale difference sequence with 
$$\E[ \| {\hat M}^{(N)}_{k+1} \|^2 \!\mid\! \F_k ] \leq \ind \{ k < \tau_N \} \cdot  K_F  ( 1 +  \| \hat x^{n_k}(\tl t(k))\|^2) \leq  K_F (1 + N^2 ),$$
where $n_k$ is such that $m(n_k) \leq k < m(n_k+1)$ (more specifically, $n_k$ is given by $n_k \= \max \{ \ell \geq 0: T_\ell \leq \tl t(k) \}$ and thus $\F_k$-measurable).
As in the proof of Lemma~\ref{lem-cvg-3}, it then follows that the sequence $\{ \zeta^{(N)}_n\}_{n \geq 0}$ given by $\zeta^{(N)}_n \= \sum_{k=0}^{n-1} \tl \alpha_k \tl \Lambda_k {\hat M}^{(N)}_{k+1}$ with $\zeta^{(N)}_0 \= 0$ is a square-integrable martingale and converges a.s.\ in $\R^{d}$ by \cite[Prop.\ VII-2-3(c)]{Nev75}. Since $\sup_{n \geq 0} \sup_{t \in [T_n, T_{n+1}]} \| \hat x^n(t) \|  < \infty$ a.s.\ by Lemma~\ref{lem-alt2}, the definitions of $\tau_N$ and $\{{\hat M}^{(N)}_k\}$ imply that almost surely, $\{\zeta_n\}_{n \geq 1}$ coincides with $\{\zeta^{(N)}_n\}_{n \geq 1}$ for some sample path-dependent value of $N$. This leads to the a.s.\ convergence of $\{\zeta_n\}$ in $\R^{d}$.
\end{proof}

Using Lemmas~\ref{lem-alt2} and \ref{lem-alt3}, we can now apply essentially the same proof steps used for \cite[Chap.\ 2, Lem.\ 1]{Bor09} to obtain that
\begin{equation}
   \lim_{n \to \infty} \sup_{t \in [T_n, T_{n+1}]} \left\| \hat x^n(t) - x^n(t) \right\| = 0 \ \ \ a.s.,
\end{equation}   
where $x^n(\cdot)$ is redefined to be the unique solution of the ODE
$$ \dot{x}(t) = \tl \lambda(t) \, h_{r(n)} (x(t))  \ \ \  \text{with} \ x^n(T_n) = \hat x(T_n) = x_{m(n)}/r(n),$$
and the function $\tl \lambda(\cdot)$ is given by (\ref{eq-tlambda}) in Section~\ref{sec-prel-ana}.
 
From this point forward, we can argue similarly to Section~\ref{sec-stab-prf2} to establish the a.s.\ boundedness of the iterates $\{x_n\}$ from algorithm (\ref{eq-alg0}). Since, in this case, as $t \to \infty$, $\tl \lambda(t + \cdot)$ converges in $\tl \Upsilon$ to the unique limit point $\bar \lambda(\cdot) \equiv \tfrac{1}{d} I$ (Lemma~\ref{lem-cvg-2}), there is no need to consider multiple limit points as in Section~\ref{sec-stab-prf2}. Consequently, the proof arguments involved are slightly simpler.

\addcontentsline{toc}{section}{References} 
\bibliographystyle{unsrt} 
\let\oldbibliography\thebibliography
\renewcommand{\thebibliography}[1]{%
  \oldbibliography{#1}%
  \setlength{\itemsep}{0pt}%
}
{\fontsize{9}{11} \selectfont
\bibliography{asyn_sa_arxiv.bib}}

\end{document}

%% file: macros-asyn-sa-arxiv.tex
\usepackage{authblk}
\usepackage{amsmath,amsthm,amssymb,mathrsfs} 
\usepackage[titletoc,title]{appendix}
\usepackage[numbers,sort&compress]{natbib}
\usepackage[hidelinks]{hyperref}
\usepackage{cleveref}

\usepackage{chngcntr}
\usepackage{titlesec}

\usepackage{lipsum}

\newcommand\blfootnote[1]{%
  \begingroup
  \renewcommand\thefootnote{}\footnote{#1}%
  \addtocounter{footnote}{-1}%
  \endgroup
}

\newtheorem{theorem}{Theorem}
\newtheorem{assumption}{Assumption}
\newtheorem{lemma}{Lemma}

\newtheorem{cor}{Corollary}
\newtheorem{remark}{Remark}
\newtheorem{myassumption}{Assumption}[section] 
\newtheorem{mylemma}{Lemma}[section]

\let\secmark\S
\usepackage{enumitem}
\setlist[itemize]{itemsep=0pt,parsep=2pt,topsep=2pt}
\setlist[enumerate]{itemsep=0pt,parsep=2pt,topsep=2pt}
\def\E{\mathbb{E}}
\def\R{\mathbb{R}}
\def\I{\mathcal{I}}
\def\F{\mathcal{F}}
\def\C{\mathcal{C}}
\def\Z{\mathcal{Z}}
\def\q{b}
\def\tl{\tilde}
\DeclareMathAlphabet{\mathbbb}{U}{bbold}{m}{n}
\newcommand{\ind}{\mathbbb{1}}
\def\={\mathop{\overset{\text{\rm \tiny def}}{=}}}

\def\nmid{\,|\,}

\def\unitS{\mathbb{S}_1}